\documentclass{article}

\usepackage{latexsym}
\usepackage{amsmath, amsfonts, amsthm, mathabx}
\usepackage{amsmath, amsfonts, amsthm}
\usepackage{epsfig}
\usepackage{stmaryrd}
\usepackage{multicol}
\usepackage{pdfsync}
\usepackage{dsfont}
\usepackage{comment}
\usepackage{algorithmic}
\usepackage{algorithm}
\usepackage{caption}
\usepackage{psfrag}
\usepackage{subfig}
\usepackage{xcolor}
\usepackage{graphics}
\usepackage{booktabs}
\usepackage{cite}

\usepackage{mathtools}

\newcommand{\la}{\left \langle}
\newcommand{\ra}{\right\rangle}

\newcommand{\norm}[1]{\left\lVert #1 \right\rVert}

\newtheorem{theorem}{Theorem}[section]

\newtheorem{lemma}[theorem]{Lemma}
\newtheorem{proposition}[theorem]{Proposition}
\theoremstyle{definition}

\newtheorem{assumption}[theorem]{Assumption}

\theoremstyle{remark}
\newtheorem{remark}[theorem]{Remark}

\numberwithin{equation}{section}

\def \d {\mathrm{d}}
\newcommand{\dd}[2]{\frac{\d #1}{\d #2}}
\newcommand{\pp}[2]{\frac{\partial #1}{\partial #2}}

\newcommand{\ol}[1]{\overline{#1}}

\newcommand{\olu}{\overline{u}}

\newcommand{\Rey}{\mathrm{Re}}

\title{PDE-constrained Models with Neural Network Terms: Optimization and Global Convergence}

\author{Justin Sirignano\footnote{Mathematical Institute at the University of Oxford, E-mail: Justin.Sirignano@maths.ox.ac.uk}, Jonathan MacArt\footnote{Department of Aerospace \& Mechanical Engineering at the University of Notre Dame, E-mail: jmacart@nd.edu}, and Konstantinos Spiliopoulos\footnote{Department of Mathematics \& Statistics at Boston University, Corresponding Author, E-mail: kspiliop@math.bu.edu}\footnote{K.S. was partially supported by the National Science Foundation (DMS 1550918, DMS 2107856) and Simons Foundation Award 672441.}}

\date{\today}

\begin{document}

\maketitle

\begin{abstract}
Recent research has used deep learning to develop partial differential equation (PDE) models in science and engineering. The functional form of the PDE is determined by a neural network, and the neural network parameters are calibrated to available data. Calibration of the embedded neural network can be performed by optimizing over the PDE. Motivated by these applications, we rigorously study the optimization of a class of linear elliptic PDEs with neural network terms. The neural network parameters in the PDE are optimized using gradient descent, where the gradient is evaluated using an adjoint PDE. As the number of parameters become large, the PDE and adjoint PDE converge to a non-local PDE system. Using this limit PDE system, we are able to prove convergence of the neural network-PDE to a global minimum during the optimization. Finally, we use this adjoint method to train a neural network model for an application in fluid mechanics, in which the neural network functions as a closure model for the Reynolds-averaged Navier--Stokes (RANS) equations. The RANS neural network model is trained on several datasets for turbulent channel flow and is evaluated out-of-sample at different Reynolds numbers.
\end{abstract}

\section{Introduction}

Consider the optimization problem where we seek to minimize the objective function
\begin{eqnarray}
J(\theta) = \frac{1}{2}  \sum_{\ell=1}^L \bigg{(} \la u, m_{\ell} \ra - \la h, m_{\ell} \ra \bigg{)}^2,
\label{ObjectivePDEIntro}
\end{eqnarray}
where $h$ is a target profile, $\{m_{\ell}\}_{\ell=1}^{L}$ are given functions that could be interpreted as moments or eigenfunctions for example. Here, we use the inner product notation $\la u, m_{\ell} \ra = \int_{U} u(x) m_{\ell}(x) dx$  to denote the inner product between two functions in $L^{2}(U)$. The variable $u$ satisfies the elliptic partial differential equation (PDE)
\begin{align}
A u(x) &= f( x, \theta ), \qquad x\in U\nonumber\\
u(x)&= 0,\qquad x\in \partial U,
\label{EllipticPDEIntro}
\end{align}
where $A$ is a standard  second-order elliptic operator and $U \subset \mathbb{R}^{d}$. \textcolor{black}{Assumption \ref{A:AssumptionTargetProfile} specifies the assumptions on the target profile $h$ and on $m_{\ell}, \ell=1,\cdots,L$ that are taken to be orthonormal functions in $L^{2}(U)$.} Thus, we wish to select the parameters $\theta$ such that the solution $u$ of the PDE (\ref{EllipticPDEIntro}) is as close as possible to the target profile $h$.

If $A^{\dagger}$ denotes the formal adjoint operator to $A$, then the adjoint PDE is
\begin{align}
A^{\dagger} \hat u(x)&=  ( \la u, m \ra - \la h, m \ra  )^{\top} m(x), \qquad x\in U\nonumber\\
\hat u(x)&= 0,\qquad x\in \partial U,
\label{AdjointPDEIntro}
\end{align}
where $m(x) = ( m_1(x), m_2(x), \ldots, m_L(x))$ \footnote{\textcolor{black}{With some slight abuse of notation if $L>1$, then $m(x)$ will always denote the vector $( m_1(x), m_2(x), \ldots, m_L(x))$ whereas $m_{i}(x)$ will denote the $i^{th}$ element of the vector. }} and $( \la u, m \ra - \la h, m \ra  )^{\top} m(x) = \sum_{\ell=1}^L ( \la u, m_{\ell} \ra - \la h, m_{\ell} \ra  ) m_{\ell}(x)$.

The optimization of (\ref{ObjectivePDEIntro}) with the PDE constraint (\ref{EllipticPDEIntro}) using the adjoint equation (\ref{AdjointPDEIntro}) is a classic PDE-constrained optimization problem. Adjoint methods have been widely used for the optimization of PDE models for engineering systems: see, for example, \cite{Brandenburg2009,Bueno,Cagnetti,Giles3,Giles0,Giles1,Giles2,Hinze2009,Jameson4,Jameson6,Jameson2,Jameson3,Jameson5,Giles4,Protas,Jameson1,Knopoff2013,Bosse2014,GaugerHambi2012,Gunther2016,Hazra2007,HazraSchultz2004,Kaland2014,Taasan1991,Taasan1995}. \cite{BergNystrom} recently numerically optimized a linear PDE with a neural network source term using adjoint methods. \cite{LES1,LES2}, \cite{Duraisamy}, \cite{Duraisamy3}, and \cite{Brenner} use adjoint methods to calibrate nonlinear PDE models with neural network terms.  \cite{BergNystrom,LES1,LES2, Duraisamy, Duraisamy3, Brenner} do not include a mathematical analysis of the adjoint method for PDE models with neural network terms.

We set the source term $f$ in equation (\ref{EllipticPDEIntro}) to be a neural network $f^{N}$ with $N$ hidden units:
\begin{eqnarray}
f^{N}(x, \theta) = \frac{1}{N^{\beta}} \sum_{i=1}^N c^i \sigma( w^i x+\eta^{i}),\label{Eq:NeuralNetwork}
\end{eqnarray}
where $(c^i, w^i, \eta^i) \in \mathbb{R} \times \mathbb{R}^d \times \mathbb{R}$, $w^i$ is a column vector, $x$ is a row vector, $\theta= \{ (c^i, w^i, \eta^i), i=1,\cdots,N   \}$, and $\beta\in(\frac{1}{2},1)$. The factor $\frac{1}{N^{\beta}}$ is a normalization.

 As the number of hidden units increase in the neural network (\ref{Eq:NeuralNetwork}), the number of parameters also increases. The neural network can represent more complex functions with larger numbers of hidden units and parameters.

The parameters $\theta$ are learned via continuous-time gradient descent:
\begin{eqnarray}
\dot{\theta}_{t}^{i} &=&  - \alpha^{N} \nabla_{\theta^{i}} J(\theta_t),
\label{GradientDescentIntro0}
\end{eqnarray}
where $\alpha^{N}$ is the learning rate and $\theta_{t}^{i}=(c^{i}_{t}, w^{i}_{t}, \eta^{i}_{t})$ for $i=1,\cdots, N$. The gradient $\nabla_{\theta^{i}} J(\theta_t)$ is evaluated by solving the adjoint PDE (\ref{AdjointPDEIntro}) and using the formula
\begin{eqnarray}
\nabla_{\theta} J(\theta) =\int_{U} \hat u(x) \nabla_{\theta} f(x, \theta) dx.
\label{GradientObjIntro0}
\end{eqnarray}

We will analyze the solution to equations (\ref{EllipticPDEIntro}), (\ref{AdjointPDEIntro}), and (\ref{GradientDescentIntro0}) when the source term is a neural network $f^N$ and as the number of hidden units $N \rightarrow \infty$.

Define
\begin{align}
g^{N}_{t}(x)&=f^{N}(x,\theta_{t})=\frac{1}{N^{\beta}} \sum_{i=1}^N c^i_{t} \sigma( w^i_{t} x+\eta^{i}_{t}).\label{Eq:NN}
\end{align}

Let $u(x; f)$ and $\hat u(x;f)$ respectively be the solutions $u(x)$ and $\hat u(x)$ with a forcing function $f$. Define $(u^{N}(t,x),\hat u^{N}(t,x))=(u(x;g^{N}_{t}),\hat u(x;g^{N}_{t}))$. Let us also define the notation $u^N(t) = u^N(t, \cdot)$ and $\hat u^N(t) =\hat u^N(t, \cdot)$.  We explicitly denote the dependence of the objective function (\ref{ObjectivePDEIntro}) on the number of hidden units $N$ in the neural network by introducing the notation
\begin{eqnarray}
J^{N}_{t} \equiv J(\theta_t)=\frac{1}{2} \norm{ \la u^N - h, m \ra }_2^2,
\label{ObjectiveFunctionN}
\end{eqnarray}
where $\norm{\cdot}_2$ is the Euclidean norm.

We prove that for any $\epsilon > 0$, there exists a time $\tau > 0$ and an $N_0\in\mathbb{N}$ such that
\begin{align}
 \sup_{t\geq \tau, N \geq N_0} \mathbb{E} J_t^N &\leq \epsilon.\nonumber
\end{align}

The proof consists of three steps. First, we prove that, as $N\rightarrow\infty$, $u^N \rightarrow u^{\ast}$ and $\hat u^N \rightarrow \hat u^{\ast}$, where $u^{\ast}$ and $\hat u^{\ast}$ satisfy a non-local PDE system. We then prove that the limit PDE system reaches a the global minimum; i.e., for any $\epsilon > 0$, there exists a time $\tau$ such that $\norm{ ( u^{\ast}, m)(\tau) - (m, h) }_2 < \epsilon$. This analysis is challenging due to the lack of a spectral gap for the eigenvalues of the non-local linear operator in the limit PDE system. The third and final step is to use the analysis of the limit PDE system to prove global convergence for the pre-limit PDE system (\ref{EllipticPDEIntro}), which requires proving that the pre-limit solution remains in a neighborhood of the limit solution on some time interval $[T_0, \infty)$.

\subsection{Applications in Science and Engineering}\label{SS:Applications}

The governing equations for a number of scientific and engineering applications are well-understood. However, numerically solving these equations can be computationally intractable for many real-world applications. The exact physics can involve complex phenomena occurring over disparate ranges of spatial and temporal scales, which can require infeasibly large computational grids to accurately resolve. Examples abound in air and space flight, power generation, and biomedicine~\cite{Curran1996,Pope2000,Wang2012,Ju2015,Sabelnikov2017}. For instance, fully resolving turbulent flows in fluid mechanics can quickly become infeasible at flight-relevant Reynolds numbers.

The exact physics equations must therefore be approximated in practice, resulting in reduced-order PDE models which can be numerically solved at tractable cost. This frequently introduces substantial approximation error, and model predictions can be unreliable and fail to match observed real-world outcomes. These reduced-order PDE models are derived from the ``exact'' physics but introduce unclosed terms which must be approximated with a model.

Examples of reduced-order PDE models in fluid mechanics include the Reynolds-averaged Navier--Stokes (RANS) equations and Large-eddy simulation (LES) equations. Recent research has explored using deep learning models for the closure of these equations \cite{Brenner, Menon1, Menon2, PeymanGivi1, Ling1,Ling2, PhysicsofFluidsLES, Duraisamy, Duraisamy3, LES1, LES2}. In parallel, other classes of machine learning methods have also been developed for PDEs such as physics-informed neural networks (PINNs) \cite{Karniadakis1,Karniadakis2, Karniadakis3, Karniadakis4, Karniadakis5, Karniadakis6, Karniadakis7}. An overview of the current progress in developing machine learning models in fluid mechanics can be found in \cite{FreundMLPersp:2019, Duraisamy2, Brunton:2020}.

This recent research on combining PDE models with neural networks motivates the analysis in our paper. \textcolor{black}{In Section~\ref{Numerical}, we first demonstrate numerically the convergence theory of this paper. We then implement the adjoint method to train a neural network RANS closure model for turbulent channel flow. }This numerical section demonstrates (A) that the present theoretical results are applicable beyond the linear case and (B) how the approach can be applied in practice to nonlinear differential equations, as well as within the framework of a strong formulation of the objective function (\ref{ObjectivePDEIntro}). The extension of our theoretical results to the nonlinear PDE case is left for future work.

\subsection{Organization of the Article}

Section~\ref{S:MainResults} presents the main theorems regarding the global convergence of the PDE solution $u^N(t,x)$ to a global minimizer of the objective function (\ref{ObjectivePDEIntro}). \textcolor{black}{Section~\ref{Numerical} includes numerical studies that demonstrate the convergence theory as well as a numerical example of neural network PDE models in fluid mechanics.} The remaining sections contain the proof of the theorems. Section~\ref{S:aprioriBounds} contains several useful \emph{a priori} bounds. The convergence $u^N \rightarrow u^{\ast}$ and $\hat u^N \rightarrow \hat u^{\ast}$ as $N \rightarrow \infty$ is proven in Section~\ref{S:ProofConvergencePDE}. The uniqueness and existence of a solution to the non-local linear PDE system is proven in Section~\ref{S:WellPosedness}. Section~\ref{S:PreliminaryCalc} studies the spectral properties of the limit PDE system and proves that $u^{\ast}(t)$ reaches a the global minimum; i.e., for any $\epsilon > 0$, there exists a time $\tau$ such that $\norm{ ( u^{\ast}, m)(\tau) - (m, h) }_2 < \epsilon$. Section~\ref{GlobalConvergencePrelimit} proves the global convergence of the pre-limit PDE solution. In Section \ref{S:Extensions}, we discuss possible extensions of this work.

\section{Main Results}\label{S:MainResults}

Let U be an open, bounded, and connected subset of $\mathbb{R}^{d}$.   We will denote by $(\cdot,\cdot)$ the usual inner product in $H=L^{2}(U)$. We shall assume that the operator $A$ is uniformly elliptic, diagonalizable and dissipative, per Assumption \ref{A:Assumption0}.

\begin{assumption} \label{A:Assumption0}
The operator $A$ is a second order uniformly elliptic operator with uniformly bounded coefficients and self-adjoint. The countable complete orthonormal basis $\{e_{n}\}_{n\in\mathbb{N}}\subset H$ that consists of eigenvectors of $A$ corresponds to a non-negative sequence $\{\lambda_{n}\}_{n\in\mathbb{N}}$ of eigenvalues which satisfy
\[
(-A)e_{n}=-\lambda_{n}e_{n}, \quad n\in\mathbb{N}
\]
and the dissipativity condition $\lambda= \displaystyle \inf_{n\in\mathbb{N}} \lambda_{n}>0$ holds.
\end{assumption}

An example of a second-order  elliptic operator is
\begin{align*}
Au(x)=-\sum_{i,j=1}^{n}\left(a^{i,j}(x)u_{x_{j}}(x)+b^{i}(x)u(x)\right)_{x_{i}}+\sum_{i=1}^{n}b^{i}(x)u_{x_{i}}(x)-c(x)u(x).
\end{align*}
The assumption above that $A$ is uniformly elliptic and self-adjoint guarantees that $A$ is diagonalizable \cite[Theorem 8.8.37]{GilbargTrudinger}.

\begin{assumption}\label{A:AssumptionTargetProfile}
We assume that the target profile $h(x)$ and the functions $m(x)$ satisfy the following conditions:
\begin{itemize}
\item The target profile $h$ satisfies $h\in L^{2}(U)$ and, in addition, $Ah\in L^{2}(U)$.
\item $\{m_{\ell}\}_{\ell=1}^{L}$ are orthonormal functions, $A m_{\ell} \in L^{2}(U)$, and $m_{\ell}\in H^{1}_{0}(U)$ for $\ell\in\{1,\cdots, L\}$.
\end{itemize}
\end{assumption}

Notice that if $L=1$, then we can consider any function $m(x)$ as long as $Am\in L^{2}(U)$ and $m\in H^{1}_{0}(U)$ since one can always normalize by $\norm{m}_{L_{2}}$.

\begin{assumption} \label{A:Assumption1} The following assumptions for the neural network (\ref{Eq:NeuralNetwork}) and the learning rate are imposed:
\begin{itemize}
\item The activation function $\sigma\in C^{2}_{b}(\mathbb{R})$, i.e. $\sigma$ is twice continuously differentiable and bounded. We also assume that $\sigma$ is non-constant.
\item The randomly initialized parameters $(C_0^i, W_0^i,\eta_{0}^{i})$ are i.i.d., mean-zero random variables with a distribution $ \mu_0(dc, dw, d\eta)$.
\item The random variable $C_0^i$ has compact support and $\la \norm{w}+|\eta|, \mu_0 \ra =\int_{\mathbb{R}^{d+2}}(\norm{w}+|\eta|)\mu_{0}(dc,dw,d\eta)< \infty$.
\item  The learning rate is chosen to be $\alpha^{N}=\frac{\alpha}{N^{1-2\beta}}$ where $0<\alpha<\infty$ and $\beta\in(1/2,1)$.
\end{itemize}
\end{assumption}

We first study the limiting behavior of the PDE system (\ref{EllipticPDEIntro})-(\ref{AdjointPDEIntro}) when the source term $f$ in (\ref{EllipticPDEIntro}) is $g^{N}_{t}(x)$. The pre-limit system of PDEs can be written as follows
\begin{align}
A u^{N}(t,x) &= g^N_{t}(x), \quad x\in U\nonumber\\
u^{N}(t,x)&= 0,\quad x\in \partial U,\label{Eq:PreLimitPDE1}
\end{align}
and
\begin{align}
A^{\dagger} \hat u^{N}(t,x)&=  ( \la u^N, m \ra - \la h, m \ra  )^{\top} m(x), \quad x\in U\nonumber\\
\hat u^{N}(t,x)&= 0,\quad x\in \partial U.\label{Eq:PreLimitPDE2}
\end{align}

$\theta_t$ satisfies the differential equation
\begin{align}
\dot{\theta}^{i}_{t}=-\alpha^{N}\nabla_{\theta^{i}_{t}}J(\theta_{t}),
\label{Eq:thetaEvolution}
\end{align}
where $\theta^{i}_{0}\sim \mu_{0}$ are sampled independently for each $i=1,\cdots, N$ and $\alpha^{N}$ is the learning rate. We will denote by $\Theta$ the space in which the initialization $\theta^{i}_{0}=(c^{i}_{0},w^{i}_{0},\eta^{i}_{0})$ take values in. The gradient is evaluated via the formula (see Lemma \ref{Eq:ObjFcnRep})
\begin{eqnarray}
\nabla_{\theta_i} J(\theta_t) =\int_{U} \hat u^{N}(t,x)\nabla_{\theta_i} f^N(x, \theta_t) dx.
\label{GradientObjIntro}
\end{eqnarray}

Before proceeding to describe the limiting PDE system, a few remarks on the pre-limit system (\ref{Eq:PreLimitPDE1})-(\ref{Eq:PreLimitPDE2}) are in order.
\begin{remark}\label{R:GlobalRegularity}
Based on our assumptions, the Lax--Milgram theorem (see for example Chapter 5.8 of \cite{GilbargTrudinger}) says that the elliptic boundary-value problem
(\ref{Eq:PreLimitPDE1})-(\ref{Eq:PreLimitPDE2}) has a unique weak solution $(u^{N},\hat{u}^{N})\in H^{1}_{0}(U)\times H^{1}_{0}(U)$. Let us also recall here that classical elliptic regularity theory guarantees that if for given $m\in\mathbb{N}$, the coefficients of the linear elliptic operator $A$ are in $\mathcal{C}^{m+1}(\bar{U})$  and $\partial U\in\mathcal{C}^{m+2}$, then the unique solution $(u^{N},\hat{u}^{N})$ to (\ref{Eq:PreLimitPDE1})-(\ref{Eq:PreLimitPDE2}) actually satisfies $(u^{N},\hat{u}^{N})\in H^{m+2}(U)\times H^{m+2}(U)$. In fact, if the coefficients of $A$ are in $\mathcal{C}^{\infty}(\bar{U})$  and $\partial U\in\mathcal{C}^{\infty}$, then $(u^{N},\hat{u}^{N})\in \mathcal{C}^{\infty}(\bar{U})\times \mathcal{C}^{\infty}(\bar{U})$.  The interested reader is referred to classical manuscripts for more details, see for example \cite[Chapter 8]{GilbargTrudinger}. Here $\bar{U}$ is the closure of the set $U$ and $\mathcal{C}^{k}$ denotes the set of $k$-times continuously differentiable functions.
\end{remark}

\begin{remark}\label{R:NonZeroBoundaryData}
The assumption of zero boundary conditions for the PDE (\ref{Eq:PreLimitPDE1})-(\ref{Eq:PreLimitPDE2}) is for notational convenience and is done without loss of generality. For example, we can consider the PDE (\ref{Eq:PreLimitPDE1}) with non-zero boundary data, say $u^{N}= q, \text{ on } \partial U$, under the assumption $q\in H^{1}(U)$ for unique solvability of (\ref{Eq:PreLimitPDE1}); see \cite[Chapter 8]{GilbargTrudinger} for more details.
\end{remark}

Let us now describe the limiting system. For a given measure $\mu$, let us set
\begin{align}
B(x,x';\mu)&= \alpha\la  \sigma(w x' + \eta) \sigma(w x + \eta)+    \sigma'(w x' + \eta)   \sigma'(w x  + \eta) c^2 (x'x+1), \mu \ra.\label{Eq:Bfcn}
\end{align}

Define the empirical measure
\begin{eqnarray}
\mu_t^N = \frac{1}{N} \sum_{i=1}^N \delta_{c^i_t, w^i_t, \eta^i_t}.\nonumber
\end{eqnarray}

As it will be shown in the proof of Lemma \ref{L:Bound_g}, see equation (\ref{Eq:Rep_g}), we can write the following representation for the neural network output
 \begin{align}
g_{t}^{N}(x)&= g_{0}^{N}(x)-\int_{0}^{t}\int_{U}\hat{u}^{N}(s,x') B(x,x';\mu^{N}_{s})dx'ds.\label{Eq:Rep_g0}
\end{align}

Consider now the non-local PDE system
\begin{align}
A u^{\ast}(t,x) &= g_{t}(x), \quad x\in U\nonumber\\
u^{\ast}(t,x)&= 0,\quad x\in \partial U,\label{Eq:LimitPDE1}
\end{align}
and
\begin{align}
A^{\dagger} \hat u^{\ast}(t,x)&= (\la u^{\ast}(t,x), m \ra - \la h, m \ra)^{\top} m(x), \quad x\in U\nonumber\\
\hat u^{\ast}(t,x)&= 0,\quad x\in \partial U.\label{Eq:LimitPDE2}
\end{align}
with
\begin{align}
g_t(x)
&=-\int_{0}^{t} \left(\int_{U} \hat u^{\ast}(s,x') B(x,x';\mu_{0}) dx'\right) ds.\label{Eq:Limiting_g}
\end{align}
The system of equations given by (\ref{Eq:LimitPDE1})-(\ref{Eq:LimitPDE2}) is well-posed as our next Theorem \ref{T:WellPosednessLimitSystem} shows. Its proof is in Section~\ref{S:WellPosedness}.
\begin{theorem}\label{T:WellPosednessLimitSystem}
Let $0<T<\infty$ be given and assume that Assumptions \ref{A:Assumption0} and \ref{A:AssumptionTargetProfile} hold. Then, there is a unique solution to the system (\ref{Eq:LimitPDE1})-(\ref{Eq:LimitPDE2}) in $C([0,T];H^{1}_{0}(U))\times C([0,T];H^{1}_{0}(U))$.
\end{theorem}

We will show that $(u^N,\hat{u}^N,g^{N}) \rightarrow (u^{\ast},\hat{u}^{\ast}, g) $ as $N\rightarrow\infty$. In particular, denoting by $\mathbb{E}$ the expectation operator associated with the randomness of the system (generated by the initial conditions), we have the following result whose proof is in Section~\ref{S:ProofConvergencePDE}.
\begin{theorem} \label{T:ConvergenceNTheorem}
Let $0<T<\infty$ be given and assume the Assumptions \ref{A:Assumption0}, \ref{A:AssumptionTargetProfile} and \ref{A:Assumption1} hold. Then, as $N \rightarrow \infty$,
\begin{eqnarray}
\max_{0 \leq t \leq T} \mathbb{E}\norm{u^N(t)-u^{\ast}(t)}_{H^{1}_{0}(U)} &\rightarrow& 0, \notag \\
\max_{0 \leq t \leq T} \mathbb{E}\norm{\hat u^N(t)-\hat u^{\ast}(t)}_{H^{1}_{0}(U)} & \rightarrow& 0, \notag \\
\max_{0 \leq t \leq T} \mathbb{E}\norm{g^N_{t} -g_{t}}_{L^{2}(U)} &\rightarrow& 0. \label{Eq:SystemConvergence}
\end{eqnarray}
In addition, $\la f, \mu^{N}_{t} \ra \overset{p} \rightarrow \la f, \mu_0 \ra$ (convergence in probability) for each $t\geq 0$ for all $f \in C^2_b(\mathbb{R}^{2 + d })$.
\end{theorem}

Let us return now to the objective function (\ref{ObjectiveFunctionN}). $J^{N}_{t}\rightarrow \bar{J}_{t}$ as $N\rightarrow\infty$ where
\begin{align}
\bar{J}_{t}&=\frac{1}{2} \norm{ \la u^{\ast}(t) - h, m \ra }_2^2,\nonumber
\end{align}

Differentiating this quantity yields
\begin{align}
\frac{\partial \bar{J}_{t}}{\partial t}&=\int_{U} \left(u^{\ast}(t,x) - h(x), m^{\top} \right) m(x) \frac{\partial u^{\ast}(t,x)}{\partial t}dx' .
\end{align}

Next notice that $\frac{\partial u^{\ast}(t,x)}{\partial t}$ is the solution to (\ref{Eq:LimitPDE1}) with $f(x,\theta)$ replaced by $\dot{g}_t(x)$. Due to equation (\ref{Eq:Limiting_g}), we also know that
\begin{align}
\dot{g}_t(x)
&=-\int_{U} \hat u^{\ast}(t,x') B(x,x';\mu_{0}) dx'.\nonumber
\end{align}

Therefore,
\begin{align}
\frac{\partial \bar{J}_{t}}{\partial t}&=\int_{U} \left(u^{\ast}(t,x) - h(x), m^{\top} \right) m(x) \frac{\partial u^{\ast}(t,x)}{\partial t}dx' \nonumber\\
&=\int_{U} A^{\dagger} \hat u^{\ast}(t,x) \frac{\partial u^{\ast}(t,x)}{\partial t}dx\nonumber\\
&=\int_{U}  \hat u^{\ast}(t,x) A\frac{\partial u^{\ast}(t,x)}{\partial t}dx\nonumber\\
&=\int_{U}  \hat u^{\ast}(t,x) \dot{g}_t(x) dx\nonumber\\
&=-\int_{U}  \hat u^{\ast}(t,x) \int_{U} \hat u^{\ast}(t,x') B(x,x';\mu_{0}) dx' dx\nonumber\\
&=-\int_{U} \int_{U}  \hat u^{\ast}(t,x) \hat u^{\ast}(t,x') B(x,x';\mu_{0}) dx' dx\nonumber\\
&=-\int_{\Theta} \left|\int_{U}   \hat u^{\ast}(t,x) \sigma(w x+\eta) dx\right|^{2} d\mu_{0}(dcdwd\eta)
-\int_{\Theta} \left|\int_{U}   \hat u^{\ast}(t,x)c \sigma'(w x+\eta)x dx\right|^{2} d\mu_{0}(dcdwd\eta)\nonumber\\
&\quad-\int_{\Theta} \left|\int_{U}   \hat u^{\ast}(t,x)c \sigma'(w x+\eta) dx\right|^{2} d\mu_{0}(dcdwd\eta).
&\leq 0\label{Eq:LyapunovFnct}
\end{align}

The calculation (\ref{Eq:LyapunovFnct}) shows that the objective function $\bar{J}_t$ is monotonically decreasing. However, it is not immediately clear
whether $u^{\ast}(t, \cdot)$ converges to a local minimum, saddle point, or global minimum. Using a spectral analysis, we are able to prove that $u^{\ast}(t,\cdot)$ will approach arbitrarily close to a global minimizer.

\begin{theorem} \label{T:GlobalConvergenceLimitPDE}
Let Assumptions \ref{A:Assumption0}, \ref{A:AssumptionTargetProfile} and \ref{A:Assumption1} hold. Then, for any $\epsilon > 0$ and $t > 0$, there exists a time
$\tau > t$ such that
\begin{eqnarray}
\bar{J}_{\tau} < \epsilon.\nonumber
\end{eqnarray}
\end{theorem}

Furthermore, we prove in Theorem \ref{T:PrelimitGlobalMinimumConvTheorem} that the original pre-limit objective function $J^N_t$ converges to a global minimum. Similar to the analysis in equation (\ref{Eq:LyapunovFnct}), one can also show that, for every $N\in\mathbb{N}$, $J^{N}_{t}$ is non-increasing in $t$. Using (\ref{Eq:Rep_g0}) and a similar calculation as in (\ref{Eq:LyapunovFnct}),
 \begin{align}
\frac{\partial J^{N}_{t}}{\partial t}
&=-\int \left|\int_{U}   \hat u^{N}(t,x) \sigma(w x+\eta) dx\right|^{2} d\mu^{N}_{t}(dcdwd\eta)
-\int \left|\int_{U}   \hat u^{N}(t,x)c \sigma'(w x+\eta)x dx\right|^{2} d\mu^{N}_{t}(dcdwd\eta)\nonumber\\
&\quad -\int \left|\int_{U}   \hat u^{N}(t,x)c \sigma'(w x+\eta) dx\right|^{2} d\mu^{N}_{t}(dcdwd\eta)\nonumber\\
&\leq 0\label{Eq:LyapunovFnct2}
\end{align}

Calculation (\ref{Eq:LyapunovFnct2}) implies that $J^{N}_{t}$ acts as a Lyapunov function for the pre-limit system and we have the following convergence result.

\begin{theorem}\label{T:PrelimitGlobalMinimumConvTheorem}
Let Assumptions \ref{A:Assumption0}, \ref{A:AssumptionTargetProfile} and \ref{A:Assumption1} hold. For any $\epsilon > 0$, there exists a time $\tau > 0$ and an $N_0\in\mathbb{N}$ such that
\begin{align}
 \sup_{t\geq \tau, N \geq N_0} \mathbb{E} J_t^N &\leq \epsilon.\nonumber
\end{align}
\end{theorem}

The proofs of Theorems \ref{T:GlobalConvergenceLimitPDE} and \ref{T:PrelimitGlobalMinimumConvTheorem} are in Sections \ref{S:PreliminaryCalc} and \ref{GlobalConvergencePrelimit}.

\begin{remark}
In our theoretical analysis we have chosen the normalization of the neural network to be $1/N^{\beta}$ with $\beta\in(1/2,1)$. We believe that similar results will also hold for the boundary cases $\beta=1/2$ and $\beta=1$, even though, the proofs presented below for the case $\beta\in(1/2,1)$ suggest that if $\beta=1/2$ or $\beta=1$ additional terms survive in the limit that require extra analysis. In this first work on this topic we decided to present the proofs for $\beta\in(1/2,1)$ and to leave the consideration of the boundary cases for future work.
\end{remark}

\section{Numerical studies}\label{Numerical}
\subsection{Numerical verification of convergence theory} \label{NumericalVerification}

In this section, we verify numerically the convergence of the neural network PDE model for the exact neural network architecture and training method that is studied in our mathematical theory. As a simple example of the class of PDEs (\ref{EllipticPDEIntro}) that we studied in the mathematically theory, we consider the heat equation $\mu u_{xx} - u = 0$ in $d = 1$ and $U = [0,1]$. The objective function (\ref{ObjectivePDEIntro}) is constructed for a target function $h(x) = (1 -x ) x^5 + x(1 - x)^5$ and the orthonormal basis $\cos( \ell \pi x)$ with $\ell = 0, \ldots, 9$. The diffusion constant $\mu = \frac{1}{10}$. We use exactly the neural network architecture (\ref{Eq:NeuralNetwork}) for which our mathematical theory is proven with a choice of $\beta = \frac{2}{3}$ for the normalization factor. The parameters are trained with standard gradient descent (again as in our theoretical setting). The gradients are evaluated by solving the adjoint equation (\ref{AdjointPDEIntro}). An implicit solver is used for fast solutions of the forward neural network-equation (\ref{EllipticPDEIntro}) and its corresponding adjoint equation (\ref{AdjointPDEIntro}).

Results are presented below for different numbers of hidden units $N$. For each value of $N$, the neural network is randomly initialized with (A) i.i.d.\ Gaussian random variables for the $(w^i, b^i)$ parameters and (B) i.i.d.\ uniform random variables for the $c^i$ parameters. This random initialization satisfies our Assumption \ref{A:Assumption1}. For each value of $N$, the solution $u(x)$ to the heat equation (\ref{Eq:NeuralNetwork}) with the trained neural network $f(x; \theta)$ is plotted in Figure \ref{fig:1DNumericalExample}. As the number of hidden units $N$ increases, the neural network PDE solution converges to the target function $h(x)$. We also evaluate the objective function (\ref{ObjectivePDEIntro}) for the trained neural network PDE as the number of hidden units is increased. Table (\ref{1D:ObjectiveFunction}) reports the value of the objective function as the number of hidden units $N$ is increased; it quickly becomes very small for larger numbers of hidden units.

\begin{figure}
  \centering
  \includegraphics[width=0.5\textwidth]{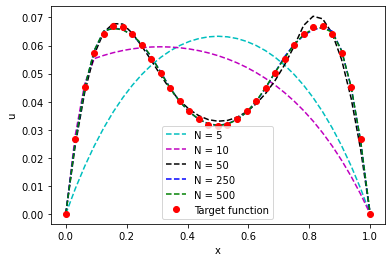}
  \caption{Solutions to (\ref{EllipticPDEIntro}) in one spatial dimension as the number of hidden units $N$ increases. The target function $h(x)$ is also displayed for comparison.}
  \label{fig:1DNumericalExample}
\end{figure}

\begin{table}
  \centering
  \begin{tabular}{c c }
    \toprule
    N & Objective Function \\
    \midrule
    5  &  $1.31 \times 10^{-4}$  \\
    10 &  $9.49 \times 10^{-5}$   \\
    50  &  $6.23 \times 10^{-7}$   \\
    250  &  $2.49 \times 10^{-10}$   \\
    500  & $3.94 \times 10^{-13}$     \\
    \bottomrule
  \end{tabular}
  \caption{Value of the objective function (\ref{ObjectivePDEIntro}) for the trained neural network PDE solution, in one spatial dimension, as the number of hidden units $N$ is increased.}
  \label{1D:ObjectiveFunction}
\end{table}

In addition, we also study the numerical convergence of the neural network PDE model for an example of (\ref{EllipticPDEIntro}) in two dimensions. Specifically, we consider the heat equation $\mu (u_{xx}  + u_{yy})- u = 0$ and $U = [0,1] \times [0,1]$
with $\mu = \frac{1}{10}$. The objective function (\ref{ObjectivePDEIntro}) is constructed for a target function $h(x,y) = (x - x^2) (y - y^2) $ and the orthonormal basis $\cos( \ell_x \pi x) \times \cos( \ell_y \pi y)$ with $\ell_x, \ell_y = 0, \ldots, 9$. The same
neural network architecture, normalization factor, gradient descent algorithm, and parameter initialization are used as in the previous one-dimensional example. A 2D implicit solver is used for fast solutions of the forward neural network-equation (\ref{EllipticPDEIntro}) and its corresponding adjoint equation (\ref{AdjointPDEIntro}). The neural network and training algorithm match the framework studied in our mathematical theory. Results are presented below for different numbers of hidden units $N$. For each value of $N$, the neural network is randomly initialized, trained with gradient descent, and then the solution $u(x)$ to the heat equation (\ref{Eq:NeuralNetwork}) with the trained neural network $f(x; \theta)$ is plotted in Figure \ref{fig:2DNumericalExample}. As the number of hidden units $N$ increases, the neural network PDE solution converges towards the target function $h(x)$. We also evaluate the objective function (\ref{ObjectivePDEIntro}) for the trained neural network PDE as the number of hidden units is increased. Table (\ref{2D:ObjectiveFunction}) reports the value of the objective function as the number of hidden units $N$ is increased; it quickly becomes very small for larger numbers of hidden units.

\begin{figure}
  \centering
  \includegraphics[width=0.32\textwidth]{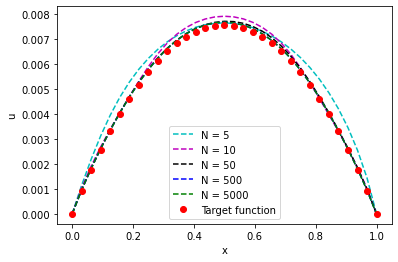}
  \includegraphics[width=0.32\textwidth]{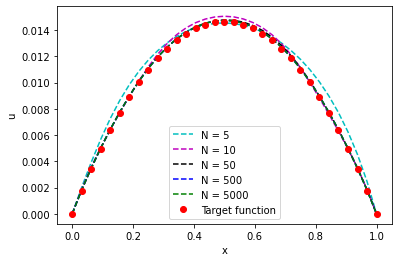}
  \includegraphics[width=0.32\textwidth]{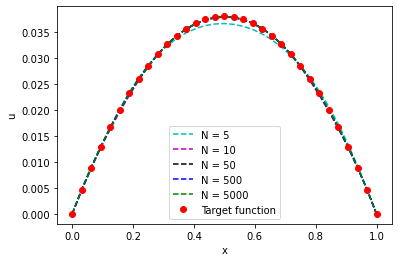}
  \includegraphics[width=0.32\textwidth]{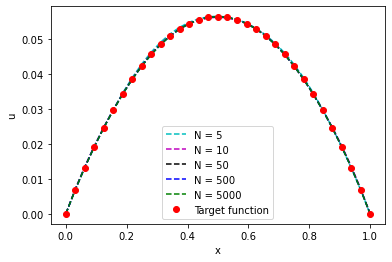}
  \includegraphics[width=0.32\textwidth]{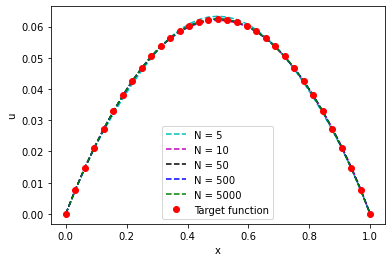}
  \includegraphics[width=0.32\textwidth]{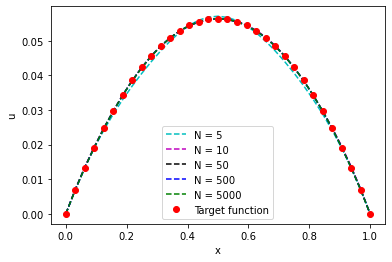}
  \includegraphics[width=0.32\textwidth]{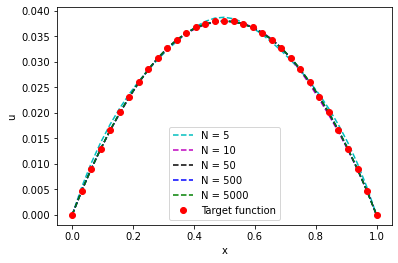}
  \includegraphics[width=0.32\textwidth]{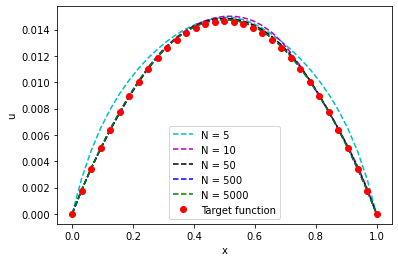}
  \includegraphics[width=0.32\textwidth]{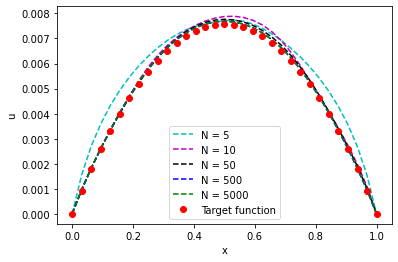}
  \caption{Solutions to (\ref{EllipticPDEIntro}) in two spatial dimensions as the number of hidden units $N$ increases. The target function $h(x)$ is also displayed for comparison. Solutions $u^N(x,y)$ are plotted for different values of $y$. In the top row: $y = 0.03125,  0.0625,$ and $0.1875$. In the middle row: $y = 0.34375, 0.5,$ and $0.65625$. In the bottom row: $y = 0.8125, 0.9375,$ and $0.96875$.}
  \label{fig:2DNumericalExample}
\end{figure}

\begin{table}
  \centering
  \begin{tabular}{c c }
    \toprule
    N & Objective Function \\
    \midrule
    5  &  $7.03 \times 10^{-8}$  \\
    10 &  $6.25 \times 10^{-9}$   \\
    50  &  $1.25 \times 10^{-9}$   \\
    500  &  $4.75 \times 10^{-10}$   \\
    5000  & $2.60 \times 10^{-10}$     \\
    \bottomrule
  \end{tabular}
  \caption{Value of the objective function (\ref{ObjectivePDEIntro}) for the trained neural network PDE solution, in two spatial dimensions, as the number of hidden units $N$ is increased.}
  \label{2D:ObjectiveFunction}
\end{table}

\subsection{Example Application: RANS Turbulence Closures} \label{Numerical2}

In this section, we demonstrate that the approach can also be applied in practice to nonlinear equations. We use a neural network for closure of the RANS equations, which leads to a nonlinear differential equation model with neural network terms. The neural network RANS model is then trained, using the adjoint equations, to match target data generated from the fully resolved physics.

The incompressible Navier--Stokes equations
\begin{align}
  \pp{u_i}{t} + u_j\pp{u_i}{x_j} &= -\frac{1}{\rho}\pp{p}{x_i} + \pp{}{x_j}\left[\nu\left(\pp{u_i}{x_j}+\pp{u_j}{x_i}\right)\right]
    \label{e:DNSmomentum} \\
    \pp{u_j}{x_j} &= 0 \label{e:DNSmass}
\end{align}
govern the evolution of a three-dimensional, unsteady velocity field $u_i(x_j,t)=(u,v,w)$,  where $x_j=(x,y,z)$ are the Cartesian coordinates. Repeated indices imply summation. The mass density $\rho$ and kinematic viscosity $\nu$ are constant, and the hydrodynamic pressure $p(x_j,t)$ is obtained from a Poisson equation~\cite{Chorin1968} that discretely enforces the divergence-free constraint~\eqref{e:DNSmass}.

Direct Numerical Simulation (DNS) solves the Navier--Stokes equations with sufficient resolution such that the turbulence may be considered fully resolved. Doing so is prohibitively expensive for high Reynolds numbers or complicated geometries. The RANS approach averages along statistically homogeneous directions (potentially including time) and so reduces the dimensionality of \eqref{e:DNSmomentum} and \eqref{e:DNSmass} with concomitant cost reduction. However, the RANS equations contain unclosed terms that require modeling.

\subsubsection{Turbulent Channel Flow}  \label{ChannelFlow}

An application to turbulent channel flow is considered here. The flow is statistically homogeneous in the streamwise direction ($x$), the spanwise direction ($z$), and time ($t$); hence the mean flow varies only in $y$. The RANS-averaged velocity components are
\begin{equation}
  \olu_i(y) = \frac{1}{L_xL_z(t_f-t_0)}\int_{t_0}^{t_f}\int_{0}^{L_z}\int_0^{L_x} u_i(x,y,z;t)\, \d x\, \d z\, \d t,
    \label{e:RANS}
\end{equation}
where $L_x$ and $L_z$ are the extents of the computational domain in the $x$- and $z$-directions, respectively, and $t_f-t_0$ is the time interval over which averaging is performed. The cross-stream and spanwise mean velocity components are identically zero: $\ol{v}=\ol{w}=0$. A schematic of the flow is shown in Figure~\ref{fig:channel}.
\begin{figure}
  \centering
  \includegraphics[width=0.75\textwidth]{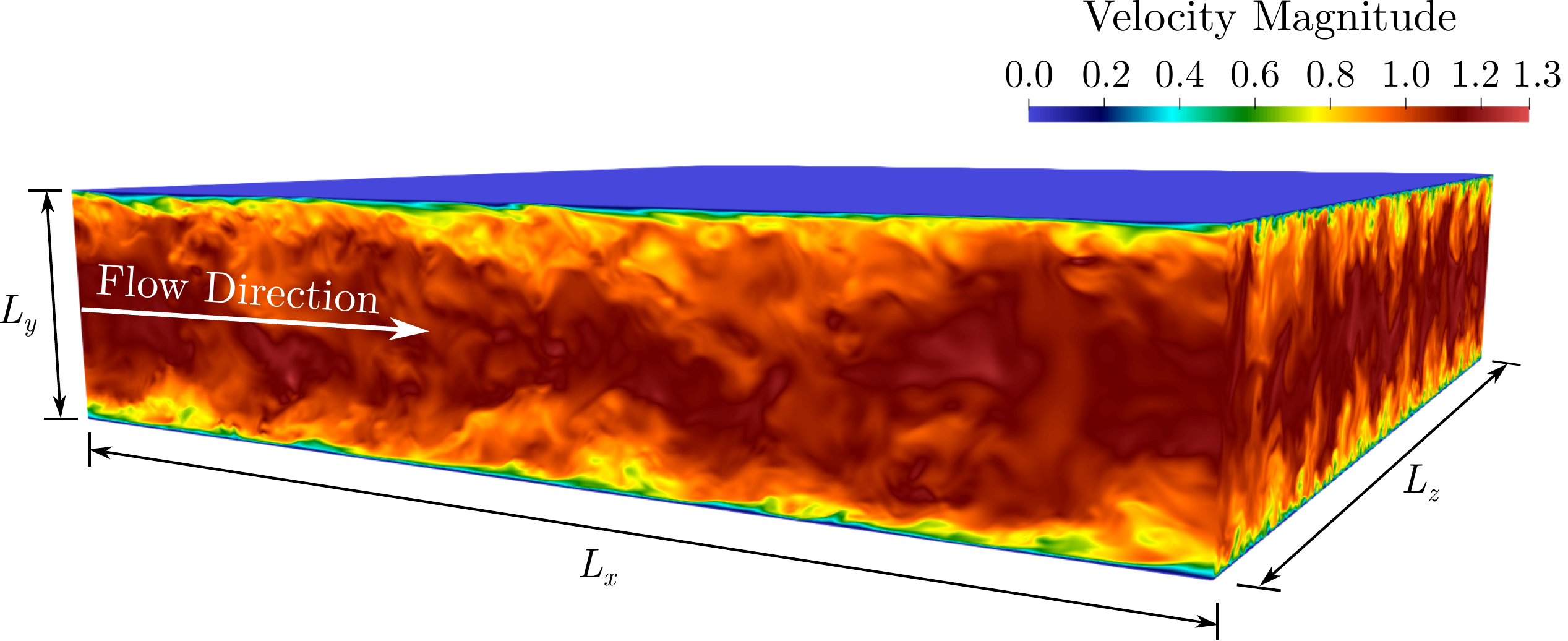}
  \caption{Schematic of the turbulent channel-flow configuration showing the domain extents and flow direction. Data is from instantaneous DNS of fully developed flow at $\Rey=18\times10^3$ (Case D).}
  \label{fig:channel}
\end{figure}

The three-dimensional, unsteady Navier--Stokes equations \eqref{e:DNSmomentum} and \eqref{e:DNSmass} are solved using DNS to provide high-resolution training and testing data. For turbulent channel flow, the characteristic length scale is the channel half-height $\delta=L_y/2$, and three characteristic velocity scales may be defined:
\begin{align*}
  \ol{U} &\equiv \frac{1}{\delta}\int_0^\delta\olu(y)\,\d y \qquad \text{(bulk velocity)}, \\
  U_0 &\equiv \olu(y=0) \qquad\quad \text{(centerline velocity)}, \\
  u_\tau &\equiv \sqrt{\tau_w/\rho} \qquad\quad\quad \text{(friction velocity)},\nonumber
\end{align*}
where $\tau(y)=\rho\nu(\d\olu/\d y)$ is the mean shear stress and $\tau_w\equiv\tau(y=-\delta)=-\tau(y=\delta)$ is the wall shear stress. These define bulk, centerline, and friction Reynolds numbers
\begin{align*}
  \Rey &\equiv \frac{2\delta\ol{U}}{\nu}, &
  \Rey_0 &\equiv \frac{\delta U_0}{\nu}, & &\text{and} &
  \Rey_\tau &\equiv \frac{u_\tau\delta}{\nu} = \frac{\delta}{\delta_v},\nonumber
\end{align*}
respectively, where $\delta_v=\nu/u_\tau$ is a viscous length scale. Flow in the near-wall boundary layer is typically presented in terms of ``wall units'' $y^+=y/\delta_v$ and the corresponding normalized velocity $u^+=\olu/u_\tau$.

\subsubsection{DNS Datasets} \label{DNS}

DNS datasets at $\Rey=\{6,9,12,18,24\}\times10^3$ were generated for training and testing. Table~\ref{tab:cases} lists the bulk, centerline, and friction Reynolds numbers for each, as well as the computed friction coefficient
\begin{equation}
  C_f\equiv \frac{\tau_w}{\frac{1}{2}\rho\ol{U}^2}.\nonumber
\end{equation}
The friction coefficient for the $\Rey_\tau=190$ case is in good agreement with that of the $\Rey_\tau=180$ DNS of Kim, Moin, and Moser~\cite{Kim1987}. Table~\ref{tab:cases} also lists the status of each dataset as it was used for training or testing. The near-wall streamwise mean velocity, in wall units, is shown for each case in Figure~\ref{fig:wall}, showing good agreement with the theoretical log-law, particularly for the higher-Reynolds-number cases.

\begin{table}
  \centering
  \begin{tabular}{c c c c c c c}
    \toprule
    Case & $\Rey$ & $\Rey_0$ & $\Rey_\tau$ & $C_f$ & Train & Test \\
    \midrule
    A   &  6,001 &  3,380 & 190.0 & $8.02\times10^{-3}$ & $\bullet$ & \\
    B   &  9,000 &  5,232 & 270.3 & $7.22\times10^{-3}$ &  & $\bullet$ \\
    C   & 12,000 &  6,800 & 345.3 & $6.62\times10^{-3}$ & $\bullet$ & \\
    D   & 18,000 & 10,191   & 495.0 & $6.05\times10^{-3}$ & $\bullet$  &  \\
    E   & 24,000 & 13,440   & 638.2 & $5.66\times10^{-3}$ & & $\bullet$ \\
    \bottomrule
  \end{tabular}
  \caption{Reynolds numbers, friction coefficients, and train/test status for channel-flow DNS cases.}
  \label{tab:cases}
\end{table}
\begin{figure}
  \centering
  \includegraphics[width=0.65\textwidth]{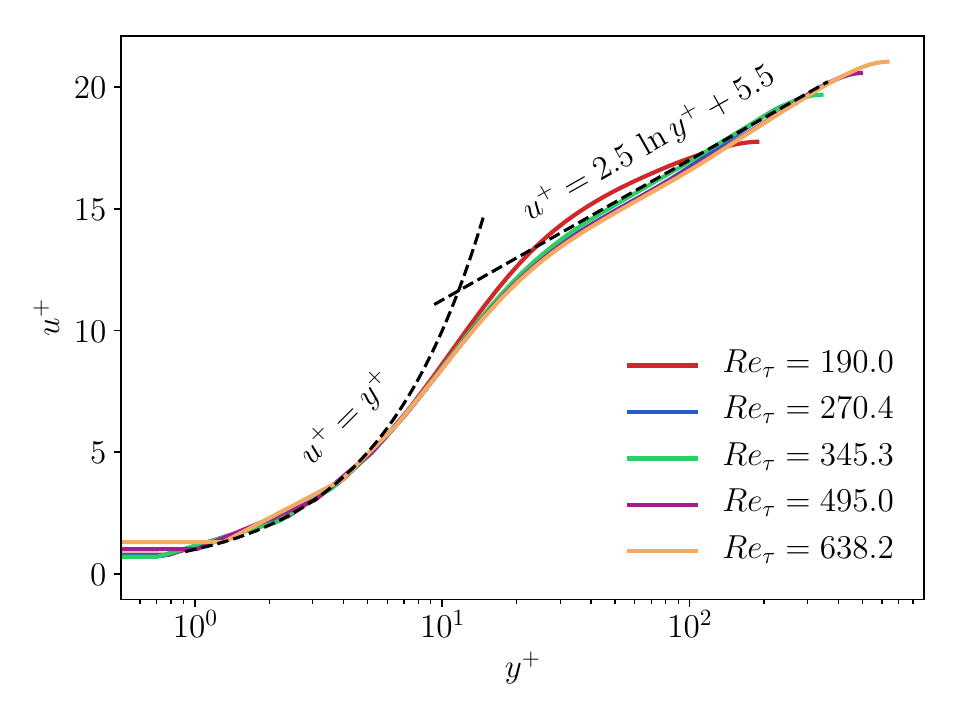}
  \caption{Mean velocity profiles in viscous units for DNS datasets listed in Table~\ref{tab:cases}. The inner and log layers are indicated by dashed lines for reference.}
  \label{fig:wall}
\end{figure}

The DNS solves the Navier--Stokes equations using second-order centered differences on a staggered mesh~\cite{Harlow1965}. Time integration is by a fractional-step method~\cite{Kim1985} with a linearized trapezoid method in an alternating-direction implicit (ADI) framework for the viscous terms~\cite{Desjardins2008}. A fixed time-step size is used, with resulting CFL numbers about 0.35.

Uniform mesh spacings $\Delta x=L_x/n_x$ and $\Delta z=L_z/n_z$ are used in the $x$- and $z$-directions, respectively, where $n_x$, $n_y$, and $n_z$ are the number of mesh cells in each direction. The mesh is stretched in the $y$-direction, with discrete cell-face points given by a hyperbolic tangent function
\begin{equation*}
  y_j = \delta\frac{\tanh\left[s_y\left(\frac{2j}{n_y}-1\right)\right]}{\tanh s_y}, \qquad j=0,1,\dotsc,n_y,
\end{equation*}
where $s_y=1.05$ is a stretching parameter. Grid resolutions, listed in Table~\ref{tab:grids}, are sufficient to place the first cell midpoint $0.71\leq\Delta y_w^+\leq1.30$ from the wall with maximum centerline grid spacing $3.61\leq\Delta y_0^+\leq6.67$. Depending on the case, the $x$- and $z$-directions are discretized with resolutions between $\Delta x^+=\Delta x/\delta_v=2.64$ and $\Delta z^+=\Delta z/\delta_v=5.39$.
\begin{table}
  \centering
  \begin{tabular}{c c c c c c c}
    \toprule
    Case & $(L_x,L_y,L_z)/\delta$ & $n_x\times n_y\times n_z$ & $\Delta y^+_w$ & $\Delta y^+_0$ & $\Delta x^+$ & $\Delta z^+$ \\
    \midrule
    A   &  $29\times2\times24$ & $1024\times128\times768$ & 0.78 & 3.99  & 5.38 & 5.94 \\
    B   &  $10\times2\times10$ & $1024\times192\times768$ & 0.74 & 3.78  & 2.64 & 3.52 \\
    C   &  $20\times2\times20$ & $1280\times256\times1280$ & 0.71 & 3.61 & 5.39 & 5.39 \\
    D   &  $10\times2\times10$ & $1280\times256\times1280$ & 1.02 & 5.19 & 3.87 & 3.87 \\
    E   &  $10\times2\times10$ & $1280\times256\times1280$ & 1.30 & 6.67 & 4.99 & 4.99 \\
    \bottomrule
  \end{tabular}
  \caption{Domain sizes, grid sizes, and resolutions in wall units for turbulent channel-flow DNS cases.}
  \label{tab:grids}
\end{table}

\subsubsection{Turbulent Channel-flow RANS}

The momentum equation solved by RANS is obtained by applying \eqref{e:RANS} to \eqref{e:DNSmomentum} and letting $t_f \rightarrow \infty$. For turbulent channel flow, it reduces to a steady, one-dimensional equation for the streamwise mean velocity component $\olu(y)$,
\begin{equation}
  0 = -\frac{1}{\rho}\dd{\ol{p}}{x} + \dd{}{y}\left(\nu\dd{\olu}{y} -
  \ol{u'v'}\right),
  \label{eq:RANS_channel}
\end{equation}
which is essentially the heat equation with a streamwise pressure-gradient forcing term $\d \ol{p}/\d x=-\tau_w/\delta$, which is adjusted to maintain constant $\ol{U} = 1$. Note that $\d\ol{p}/\d x$ is uniform in $y$. The term $\ol{u'v'}$ (the Reynolds stress) represents the action of the velocity fluctuations $u_i'=u_i-\olu_i$ on the mean flow and is unclosed. A standard closure model used for comparison is summarized in Sec.~\ref{sec:ke}, and the neural network model to be tested is described in Sec.~\ref{RANSneuralNetwork}.

\subsubsection{The $k$-$\epsilon$ Model} \label{sec:ke}

The eddy-viscosity (Boussinesq) hypothesis enables one to relate the Reynolds stress to the mean velocity gradients and so close \eqref{eq:RANS_channel}. For turbulent channel flow, the eddy-viscosity hypothesis is stated
\begin{equation}
  \ol{u'v'} = -\nu_t\dd{\olu}{y},
  \label{eq:eddyvisc}
\end{equation}
where $\nu_t$ is the eddy viscosity to be computed. Applying \eqref{eq:eddyvisc} to \eqref{eq:RANS_channel}, the RANS equation is
\begin{equation}
  0 = -\frac{1}{\rho}\pp{\ol{p}}{x} + \dd{}{y}\left[(\nu + \nu_t)\dd{\olu}{y}\right],
  \label{eq:RANS_channel_ev}
\end{equation}
which is closed.

The eddy viscosity may be obtained in many ways. It is likely that the most common approach for engineering applications is the standard $k$-$\epsilon$ model~\cite{Jones1972,Launder1974},  which is a ``two-equation'' model in that it requires the solution of two additional PDEs. It computes $\nu_t$ from
\begin{equation}
  \nu_t = C_\mu\frac{k^2}{\epsilon},
  \label{eq:nut}
\end{equation}
where  $k\equiv\frac{1}{2}\ol{u_i'u_i'}$ is the turbulent kinetic energy (TKE), $\epsilon$ is the TKE dissipation rate, and $C_\mu$ is a model parameter. The TKE and its dissipation rate may be computed from DNS and experimental data but are unclosed in RANS. Instead, the $k$-$\epsilon$ model obtains them from PDEs with modeled terms, which for channel flow reduce to
\begin{align}
  0 &= \dd{}{y}\left(\frac{\nu+\nu_t}{\sigma_k}\dd{k}{y}\right) +
  \nu_t \left(\dd{\olu}{y}\right)^2 - \epsilon, \notag \\
    0 &= \dd{}{y}\left(\frac{\nu+\nu_t}{\sigma_\epsilon}\dd{\epsilon}{y}\right) + C_{1\epsilon} C_{\mu}  k \left(\dd{\olu}{y}\right)^2 - C_{2\epsilon}\frac{\epsilon^2}{k}.
   \label{EpsilonandKequations}
\end{align}
Typical values for the constants are~\cite{Launder1974}:
\begin{align*}
  C_\mu &= 0.09, & \sigma_k&=1.00, & \sigma_\epsilon&=1.30, & C_{1\epsilon}&=1.44, \textrm{\phantom{....} and} & C_{2\epsilon}&=1.92.
\end{align*}

\subsubsection{Neural Network Closure} \label{RANSneuralNetwork}

Our neural network RANS model augments the eddy-viscosity model \eqref{eq:eddyvisc} by
\begin{equation}
  \ol{u'v'} = -\nu_t\dd{\olu}{y} + f_1\left(y, k, \epsilon, \frac{\d \olu}{\d y}, \theta \right) +  \frac{C_R}{\Rey} f_2\left(y, k, \epsilon, \frac{\d \olu}{\d y}, \theta \right),
  \label{eq:ke_NN}
\end{equation}
where $f(\cdot, \theta )$ is a neural network with a vector output, $f_i(\cdot, \theta)$ is the $i$-th output of the neural network, and $\theta$ are its parameters. The constant $C_R = 6,001$ is the lowest Reynolds number in the training dataset. The input variables for the neural network are $(y, k, \epsilon, \d \olu/\d y)$.

 The RANS equation with the eddy-viscosity model and the neural network is
\begin{equation}
  0 = -\frac{1}{\rho}\dd{\ol{p}}{x} + \dd{}{y}\left[(\nu + \nu_t)\dd{\olu}{y} - f_1\left(y, k, \epsilon, \frac{\d \olu}{\d y}, \theta \right) -  \frac{C_R}{\Rey} f_2\left(y, k, \epsilon, \frac{\d \olu}{\d y}, \theta \right) \right],
  \label{eq:RANS_channel_NN}
\end{equation}
with $\nu_t$, $k$, and $\epsilon$ obtained as above.

The optimization of the PDE model (\ref{eq:RANS_channel_NN}) is a nonlinear version of our mathematical framework (\ref{EllipticPDEIntro}). The objective function for the model is the mean-squared error for the velocity profile (which can be thought of as a strong formulation of the objective function (\ref{ObjectivePDEIntro})):
\begin{equation}
J(\theta) = \frac{1}{2} \int_{-\delta}^\delta \bigg{(} h(y) - \olu(y) \bigg{)}^2 dy,
\end{equation}
where the target profile $h(y)$ is the averaged DNS data, and $\olu$ is the solution to (\ref{eq:RANS_channel_NN}). The operator $A$ from (\ref{EllipticPDEIntro}) has, in this case, become a nonlinear operator $A$ acting on $(\olu, k, \epsilon)$ and corresponding to the right-hand side of the equations (\ref{eq:RANS_channel_NN}) and (\ref{EpsilonandKequations}). Specifically,
\begin{eqnarray}
A \begin{pmatrix} \olu \\ k \\ \epsilon \end{pmatrix} &=& \begin{pmatrix}   -\frac{1}{\rho}\dd{\ol{p}}{x} + \dd{}{y} [(\nu + \nu_t)\dd{\olu}{y} ] \\
 \dd{}{y}\left(\frac{\nu+\nu_t}{\sigma_k}\dd{k}{y}\right) +
    \nu_t (\dd{\olu}{y})^2- \epsilon  \\
 \dd{}{y}\left(\frac{\nu+\nu_t}{\sigma_\epsilon}\dd{\epsilon}{y}\right) + C_{1\epsilon} C_{\mu}  k (\dd{\olu}{y})^2  - C_{2\epsilon}\frac{\epsilon^2}{k} \end{pmatrix},\nonumber
\end{eqnarray}
where we recall that the forcing term $\d \ol{p}/\d x=-\tau_w/\delta$ plus an adaptive term to enforce $\ol{U} \approx 1$, where $\tau_w=\rho\nu(\d\olu/\d y)_{y=-\delta}$ is the wall shear stress defined in Subsection \ref{ChannelFlow}.

\subsubsection{Numerical Results} \label{NumericalResults}

The RANS equations are discretized using second-order centered differences on a staggered mesh. The solution $\olu$ and its adjoint are computed at the $y$-faces, and the model variables $k$ and $\epsilon$ are computed at cell centers. A coarse grid of $n_y=32$ mesh points is used for all Reynolds numbers. The grid is nonuniform and is generated by the transformation
\begin{equation}
  \tilde{y} = \frac{\sin(\eta \pi \hat{y}/2 )}{ \sin( \eta \pi/2)},
\end{equation}
where $\hat{y}\in[-\delta,\delta]$ is a uniform grid and $\eta = 0.995$.  The system is solved using a pseudo-time-stepping scheme, in which the velocity solution is relaxed from an initially parabolic (laminar) profile using the fourth-order Runge--Kutta method until the solution has converged. Thus, although the RANS equations for channel flow satisfy a system of nonlinear differential equations, their solution is obtained by solving a time-dependent PDE.

We train the neural network-augmented closure model \eqref{eq:ke_NN} on the DNS-evaluated mean velocity for $\Rey=\{6, 12, 18 \} \times 10^{3}$. The adjoint of \eqref{eq:RANS_channel_NN} is constructed using automatic differentiation as implemented by the Pytorch library~\cite{Pytorch}. 100 relaxation steps are taken per optimization iteration.  Efficient, scalable training is performed using MPI parallelization and GPU offloading. The learning rate is decayed during training according to a piecewise constant learning rate schedule. Quadrature on the finite difference mesh is used to approximate the objective function.

The neural network architecture is similar to that used for \cite{LES2} and \cite{LES1}. Exponential linear units (ELUs) and sigmoids are used for the activation functions. Note that ELUs violate the assumption on boundedness of the activation function under which our theoretical results hold. Nevertheless, this example demonstrates that the mathematical results of this paper may indeed be true under less stringent assumptions. The neural network architecture is:

\begin{eqnarray}
H^{(1)}(z;\theta) &=& \textrm{ELU}( W^{(1)} z + b^{(1)} ), \notag \\
H^{(2)}(z;\theta) &=& \textrm{ELU}( W^{(2)} H^{(1)}(z;\theta) + b^{(2)} ), \notag \\
G(z;\theta) &=& \sigma( W^{(3)} z + b^{(3)} ), \notag \\
U(z;\theta) &=& G(z;\theta) \odot H^{(2)}(z;\theta), \notag \\
f(z; \theta) &=& W^{(4)} U(z;\theta) + b^{(4)},
\end{eqnarray}
where the parameters $\theta = \{ W^{(1)}, b^{(1)}, W^{(2)}, b^{(2)}, W^{(3)}, b^{(4)} \}$, $\odot$ denotes the element-wise product, $\textrm{ELU}(\cdot)$ is an element-wise ELU activation function, and $\sigma(\cdot)$ is an element-wise sigmoid activation function.


Once trained, the neural network-augmented model \eqref{eq:ke_NN} is tested out-of-sample at  $\Rey=\{9, 24\} \times 10^{3}$. We observed that if the model was optimized for a very large number of iterations, instability could occur (presumably due to the equations becoming very stiff near the walls) or it could overfit on the test cases. The results presented are for a model trained with a moderately large number of optimization iterations which produces both good performances on the training cases as well as the test cases. The results provide a useful numerical example of this paper's optimization method applied to nonlinear equations, but the challenges with overfitting on out-of-sample Reynolds numbers suggest that perhaps a PDE model with more physics (such as large-eddy simulation) should be combined with a neural network closure model instead of using RANS. The training mean-squared error is $\sim 1.8 \times 10^{-3}$. Figures~\ref{FigureResultsOutofSample} and~\ref{fig:wall_RANS} compare the mean velocity in mean-flow and wall units, respectively, for the DNS, RANS with the standard $k$-$\epsilon$ model, and RANS with the neural network-augmented model. The neural network model improves the accuracy of the $k$-$\epsilon$ model. In particular, the near-wall prediction (Figure~\ref{fig:wall_RANS}) is improved by the neural network but still does not exactly match the DNS, as might be anticipated for the coarse grid.
\begin{figure}
\centering
\includegraphics[scale=0.5]{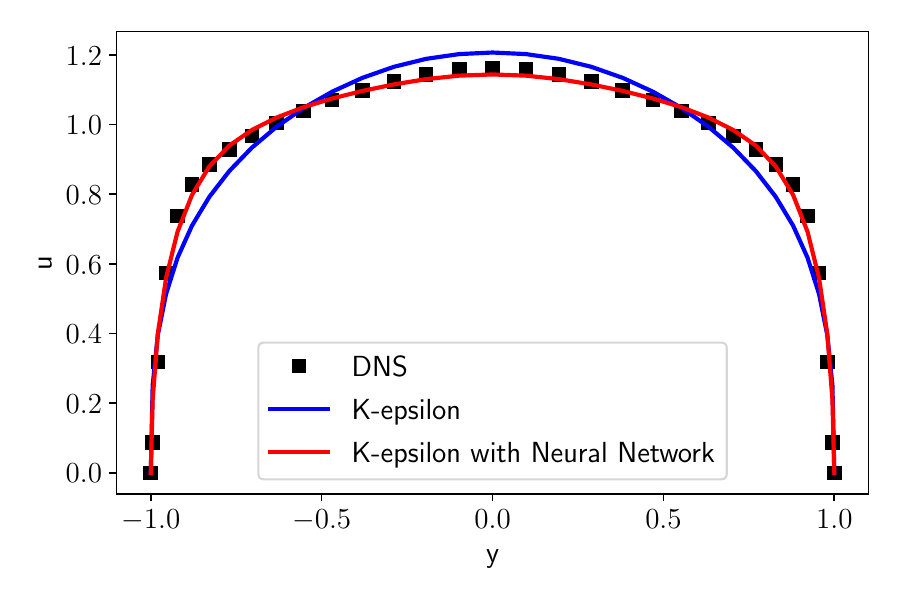}
\includegraphics[scale=0.5]{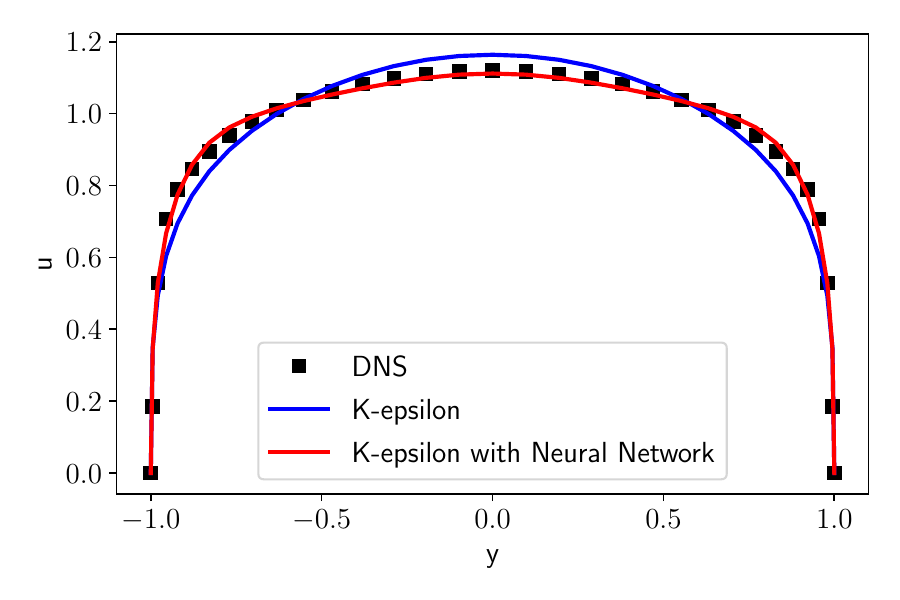}
\caption{Mean velocity profiles obtained from spatio-temporally averaged DNS, RANS using the standard $k$-$\epsilon$ model, and RANS using the neural network-augmented model. DNS were computed using the meshes in Table~\ref{tab:grids}; points indicate the cell-face locations used for the RANS calculations. Left: $\Rey= 9\times10^3$; right: $\Rey=24\times10^3$. The neural network model was trained on $\Rey=\{6, 12, 18 \} \times 10^{3}$.}
\label{FigureResultsOutofSample}
\end{figure}
\begin{figure}
  \centering
  \includegraphics[width=0.65\textwidth]{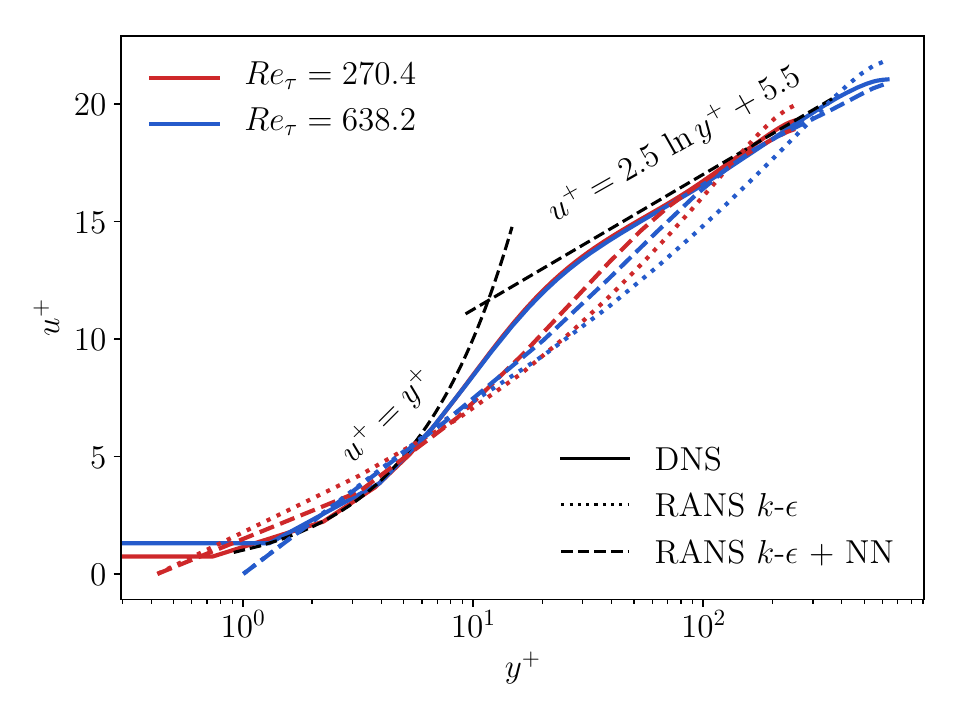}
  \caption{Mean velocity profiles in viscous units for out-of-sample testing cases: $\Rey=\{9,24\}\times10^3$. Solid lines: DNS; dotted lines: RANS with the $k$-$\epsilon$ model; dashed lines: RANS with the neural network-augmented model. The log-law is indicated by dashed lines for reference. The friction velocity used to normalize each $\Rey_\tau$ was computed from DNS. The solid lines for the two DNS datasets nearly overlap.}
  \label{fig:wall_RANS}
\end{figure}

\section{A-priori preliminary bounds}\label{S:aprioriBounds}

Let us start by obtaining certain crucial a-priori bounds. The gradient of the objective function (\ref{ObjectivePDEIntro}) can be evaluated using the solution $\hat u$ to the adjoint PDE (\ref{AdjointPDEIntro}). Indeed, this is the statement of  Lemma \ref{Eq:ObjFcnRep} below.
\begin{lemma}\label{Eq:ObjFcnRep}
Let Assumptions \ref{A:Assumption0}, \ref{A:AssumptionTargetProfile} hold. Then,
\begin{eqnarray}
\nabla_{\theta} J(\theta) =\int_{U} \hat u(x;f) \nabla_{\theta} f(x, \theta) dx
\label{GradientObj}
\end{eqnarray}
\end{lemma}
\begin{proof}[Proof of Lemma \ref{Eq:ObjFcnRep}]
We start by defining $\tilde u = \nabla_{\theta} u$. If we differentiate (\ref{EllipticPDEIntro}) we obtain
\begin{align}
A \tilde u(x) &= \nabla_{\theta} f( x,\theta ), \quad x\in U\nonumber\\
\tilde u(x)&= 0,\quad x\in \partial U\nonumber
\end{align}

Integration by parts  gives
\begin{eqnarray*}
( \hat u, A \tilde u ) &=&  ( A^{\dagger} \hat u, \tilde u ).
\end{eqnarray*}

Using now (\ref{EllipticPDEIntro}) and (\ref{AdjointPDEIntro}) for the left and right hand side we obtain
\begin{eqnarray*}
(\la u, m \ra - \la h, m \ra)^{\top} (  \tilde u, m) &=& ( \hat u, \nabla_{\theta} f(\theta) ).
\end{eqnarray*}

Due to (\ref{ObjectivePDEIntro}), we finally obtain
\begin{eqnarray}
\nabla_{\theta} J(\theta) &=& \nabla_{\theta} \bigg{[} \frac{1}{2}  ( \la u, m \ra - \la h, m \ra )^{\top}  ( \la u, m \ra - \la h, m \ra )  \bigg{]} \notag \\
&=&  ( \la u, m \ra - \la h, m \ra )^{\top}  \la \tilde u, m \ra   \notag \\
&=& ( \hat u, \nabla_{\theta} f(\theta) ),\
\end{eqnarray}
which concludes the proof of the lemma.
\end{proof}

Lemma \ref{Eq:ObjFcnRep} shows that the gradient descent update can be written as follows:
\begin{eqnarray}
\dot{\theta}_{t}^{i} &=& - \alpha^{N} \nabla_{\theta^{i}_{t}} J(\theta_t) \notag \\
&=& - \alpha^{k,N}  \left( \hat u^{N}(t,x), \nabla_{\theta^{i}_{t}} f(x, \theta_t) \right)_{L^{2}(U)},\label{SGDupdates}
\end{eqnarray}

It can be proven that the gradient descent updates are uniformly bounded with respect to $t,N$ for $t\leq T$ for a given $T<\infty$.

\begin{lemma}\label{L:C_Bound}
Let  $T<\infty$ be given and Assumptions \ref{A:Assumption0}-\ref{A:Assumption1} hold. There exists a finite constant $C<\infty$, depending on $T$ and on the operator $A$, such that for all $i\in\mathbb{N}$ and all $t\in[0,T]$,
\begin{align}
\sup_{i\in\{1,\cdots, N\}, N\in\mathbb{N},t\in[0,T]}| c_t^i | &\leq C < \infty, \notag\nonumber\\
\sup_{i\in\{1,\cdots, N\}, N\in\mathbb{N},t\in[0,T]} \mathbb{E}\left(\norm{ w_{t}^{i} }+| \eta_t^i |\right) &\leq C < \infty.\nonumber
\end{align}
\end{lemma}
\begin{proof}[Proof of Lemma \ref{L:C_Bound}]

We start with a bound on $\norm{\hat{u}^{N}(t)}_{H^{1}_{0}(U)}$. Using standard elliptic estimates based on the implied coercivity by Assumption \ref{A:Assumption0} and the Cauchy--Schwarz inequality, we obtain that there is a constant $C<\infty$ such that

\begin{align}
\norm{\hat{u}^{N}(t)}^{2}_{H^{1}_{0}(U)}&\leq \norm{ ( \la u^N(t), m \ra - \la h, m \ra )^{\top} m }_{L^{2}(U)}^2  \nonumber \\
&\leq C \norm{\la u^N(t), m \ra - \la h, m \ra }_2^2 \norm{  m }_{L^{2}(U)}^2  \nonumber \\
&\leq C\left[\norm{u^{N}(t)}^{2}_{L^{2}(U)}+\norm{h}^{2}_{L^{2}(U)}\right]\nonumber\\
&\leq C\left[\norm{u^{N}(t)}^{2}_{H^{1}_{0}(U)}+\norm{h}^{2}_{L^{2}(U)}\right]\nonumber\\
&\leq C\left[\norm{g^{N}_{t}}^{2}_{L^{2}(U)}+\norm{h}^{2}_{L^{2}(U)}\right]\nonumber\\
&\leq C\left[\frac{1}{N^{2\beta}}\int_{U}N\sum_{i=1}^{N}|c^{i}_{t}|^{2}|\sigma(w^{i}_{t}x+\eta_t^i)|^{2}dx+\norm{h}^{2}_{L^{2}(U)}\right]\nonumber\\
&\leq C\left[\frac{1}{N^{2\beta-2}}\frac{1}{N}\sum_{i=1}^{N}|c^{i}_{t}|^{2}+\norm{h}^{2}_{L^{2}(U)}\right],\nonumber
\end{align}
where in the last line we used the uniform boundedness of $\sigma$. Let us set $\kappa_{t}^{N}= \displaystyle \frac{1}{N}\sum_{i=1}^{N}|c^{i}_{t}|^{2}$. The last display gives us
\begin{align}
\norm{\hat{u}^{N}(t)}^{2}_{H^{1}_{0}(U)}&\leq  C\left[\frac{1}{N^{2\beta-2}}\kappa_{t}^{N}+\norm{h}^{2}_{L^{2}(U)}\right].\nonumber
\end{align}

Next, we calculate
\begin{align}
c^{i}_{t}&=c^{i}_{0}-\frac{\alpha}{N^{1-\beta}}\int_{0}^{t}\int_{U}\hat{u}^{N}(s,x)\sigma(w^{i}_{s} x+\eta_s^i)dxds,\nonumber
\end{align}
which then implies due to the assumed uniform boundedness of $\sigma$ that
\begin{align}
|c^{i}_{t}|&\leq |c^{i}_{0}|+\frac{\alpha\norm{\sigma}}{N^{1-\beta}} \int_0^t \int_{U}|\hat{u}^{N}(s,x)|dx ds.\nonumber
\end{align}

This then gives
\begin{align}
|c^{i}_{t}|^{2}&\leq 2 |c^{i}_{0}|^{2}
+\frac{C}{N^{2(1-\beta)}}t\int_{0}^{t}\norm{\hat{u}^{N}(s)}^{2}_{L^{2}(U)}ds\nonumber\\
&\leq 2 |c^{i}_{0}|^{2}
+\frac{C}{N^{2(1-\beta)}}t\int_{0}^{t}\left[\frac{1}{N^{2\beta-2}}\kappa_{s}^{N}+\norm{h}^2_{L^{2}(U)}\right]ds\label{Eq:C_estimate}
\end{align}

Normalizing  by $N$ and summing up over all $i=1,\cdots, N$ gives
\begin{align}
\kappa_{t}^{N}&\leq 2 \kappa_{0}^{N}
+\frac{C}{N^{2(1-\beta)}}t\int_{0}^{t}\left[\frac{1}{N^{2\beta-2}}\kappa_{s}^{N}+\norm{h}^2_{L^{2}(U)}\right]ds\nonumber\\
&\leq 2 \kappa_{0}^{N}
+Ct\int_{0}^{t}\kappa_{s}^{N}ds+\frac{C t^{2}}{N^{2-2\beta}}\norm{h}^2_{L^{2}(U)}\nonumber
\end{align}

Since $c^{i}_{0}$ take values in a compact set, there exists a constant $C<\infty$ such that $\kappa_{0}^{N}\leq C$. Hence, using the discrete Gronwall lemma we get that there is a constant (still denoted by) $C$ that may depend on $T$, but not on $N$, such that
\begin{align}
\sup_{t\in[0,T]}\kappa_{t}^{N}&\leq C.\nonumber
\end{align}

Combining the latter with (\ref{Eq:C_estimate}) gives for some unimportant constant $C<\infty$
\begin{align}
|c^{i}_{t}|^{2}
&\leq 2|c^{i}_{0}|^{2}+ C+ \frac{C}{N^{2(1-\beta)}}\norm{h}^2_{L^{2}(U)},
\end{align}
from which the result follows, concluding the bound for $|c^{i}_{t}|$.
Now, we turn to the bound for $\norm{ w^i_t }$. Using the uniform bound for $|C^{i}_{k}|$ and Assumptions \ref{A:Assumption0}-\ref{A:Assumption1} on the boundedness of the domain $U$ and on $\sigma'$ we get
\begin{align}
\norm{w^{i}_{t}}&\leq\norm{w^{i}_{0}}+\frac{\alpha}{N^{1-\beta}}\int_{0}^{t}\int_{U}|\hat{u}^{N}(s,x)||c^{i}_{s}|\sigma'(w^{i}_{s} x+\eta_s^i)x|dxds\nonumber\\
&\leq\norm{w^{i}_{0}}+\frac{C}{N^{1-\beta}}\int_{0}^{t}\norm{\hat{u}^{N}(s)}_{L^{1}(U)}|c^{i}_{s}|ds\nonumber\\
&\leq \norm{w^{i}_{0}}+ \frac{C}{N^{1-\beta}}
\int_{0}^{t}\left[\frac{1}{N^{\beta-1}}+\norm{h}_{L^{2}(U)}\right]ds\nonumber\\
&\leq \norm{w^{i}_{0}}+ C\left[1+\frac{1}{N^{1-\beta}}  \right],\nonumber
\end{align}
for a constant $C<\infty$ that may change from line to line. From this we directly obtain that
\begin{align*}
\sup_{i\in\{1,\cdots,N\}, N\in\mathbb{N}, t\in[0,T]}\mathbb{E}\norm{ w^i_t}\leq C<\infty. 
\end{align*}
The proof for the bound of $\sup_{i\in\{1,\cdots,N\}, N\in\mathbb{N}, t\in[0,T]}\mathbb{E}|\eta^i_t|$ is exactly the same as the uniform bound for $\mathbb{E}\norm{ w^i_t}$, concluding the proof of the lemma.
\end{proof}

Next, in Lemma \ref{L:Bound_g}, we show an important a-priori bound relating the neural network output $g^{N}$ with the solution to the adjoint equation $\hat{u}^{N}$.
\begin{lemma}\label{L:Bound_g}
 For every $T<\infty$, we have that there is some constant $C<\infty$ (independent of N) such that for all  $t\in[0,T]$ such that
\begin{align}
\max_{s \in[0,t]} \mathbb{E}  \bigg{[} \norm{g_{s}^N}^{2}_{L^{2}(U)}  \bigg{]} &\leq  C \bigg{(} 1+\int_{0}^{t} \max_{\tau \in[0,s]} \mathbb{E} \bigg{[} \norm{\hat{u}^{N}(\tau)}^{2}_{L^{2}(U)} \bigg{]} ds \bigg{)}.\nonumber
\end{align}

\end{lemma}
\begin{proof}[Proof of Lemma \ref{L:Bound_g}]
Starting with $g_{t}^{N}(x)= \displaystyle \frac{1}{N^{\beta}} \sum_{i=1}^N c^i_{t} \sigma( w^i_{t} x+\eta^i_t)$ and differentiating using (\ref{Eq:thetaEvolution}) we obtain
\begin{align}
\dot{g}_{t}^{N}(x)&=\frac{1}{N^{\beta}} \sum_{i=1}^N \left[\dot{c}^i_{t} \sigma( w^i_{t} x+\eta^i_t)+c^{i}_{t}\sigma'(w^i_{t} x+\eta^i_t)(\dot{w}^i_{t} x+\dot{\eta}^i_t)\right]\nonumber\\
&=-\frac{\alpha}{N}\sum_{i=1}^{N}\left[\int_{U}\hat{u}^{N}(t,x')\left(\sigma(w^i_{t} x+\eta^i_t)\sigma(w^i_{t} x'+\eta^i_t)+|c^{i}_{t}|^{2}\sigma'(w^i_{t} x+\eta^i_t)\sigma'(w^i_{t} x'+\eta^i_t)(x x'+1)\right)dx'\right]\nonumber
\end{align}

Recalling now the empirical measure $\mu^{N}_{t}= \displaystyle \frac{1}{N}\sum_{i=1}^{N}\delta_{c^{i}_{t}, w^{i}_{t},\eta^i_t}$ and the definition of $B(x,x';\mu)$ by (\ref{Eq:Bfcn}) we obtain the representation
\begin{align}
\dot{g}_{t}^{N}(x)&=-\int_{U}\hat{u}^{N}(t,x') B(x,x';\mu^{N}_{s})dx'\label{Eq:Rep_g}
\end{align}

The latter implies that
\begin{align}
\norm{g^{N}_{t}}^{2}_{L^{2}(U)}&\leq \norm{g^{N}_{0}}^{2}_{L^{2}(U)}+\int_{U}\left(\int_{0}^{t}\int_{U}\hat{u}^{N}(s,x')B(x,x';\mu^{N}_{s})dx'ds\right)^{2}dx\nonumber
\end{align}

Notice now that for an appropriate constant $C$ (independent of $N$) that may change from line to line
\begin{align}
|B(x,x';\mu^{N}_{s})|&\leq\alpha \frac{1}{N}\sum_{i=1}^{N}\left|\sigma(w^i_{s} x+\eta^i_s)\sigma(w^i_{s} x'+\eta^i_s)+|c^{i}_{s}|^{2}\sigma'(w^i_{s} x+\eta^i_s)\sigma'(w^i_{s} x'+\eta^i_s)(x x'+1)\right|\nonumber\\
&\leq C\left[1+\frac{1}{N}\sum_{i=1}^{N}|c^{i}_{s}|^{2}\right]\nonumber\\
&\leq C,\nonumber
\end{align}
due to the boudnedness of $\sigma,\sigma', U$ and by Lemma \ref{L:C_Bound}.    Thus, we have for an appropriate constant $C$ that
\begin{align}
\norm{g^{N}_{t}}^{2}_{L^{2}(U)}&\leq \norm{g^{N}_{0}}^{2}_{L^{2}(U)}+C\int_{U}\left(\int_{0}^{t}\int_{U}|\hat{u}^{N}(s,x')|dx'ds\right)^{2}dx\nonumber\\
&\leq \norm{g^{N}_{0}}^{2}_{L^{2}(U)}+C\left(\int_{0}^{t}\norm{\hat{u}^{N}(s)}_{L^{2}(U)}ds\right)^{2}\nonumber
\end{align}

Because now $c_{0}^{i},w_{0}^{i},\eta^i_0$ are sampled independently, $c_{0}^{i}$ take values in a compact set and $\sigma$ is uniformly bounded  we have
\begin{align}
\mathbb{E}\norm{g^{N}_{0}}^{2}_{L^{2}(U)} &=\frac{1}{N^{2\beta}}\int_{U}\mathbb{E}\left(\sum_{i=1}^{N}c^{i}_{0}\sigma(w^{i}_{0}x+\eta^i_0)\right)^{2}dx\nonumber\\
&=\frac{1}{N^{2\beta}}\int_{U}\mathbb{E}\sum_{i=1}^{N}|c^{i}_{0}\sigma(w^{i}_{0}x+\eta^i_0)|^{2}dx\nonumber\\
&\leq \frac{C}{N^{2\beta-1}}\nonumber\\
&\leq C,\label{Eq:Boundg0}
\end{align}
where the last line follows because $\beta\in(1/2,1)$.

Therefore, there is a constant $C$ that is independent of $N$ such that
\begin{align}
\max_{s \in[0,t]}\mathbb{E}\norm{g^{N}_{s}}^{2}_{L^{2}(U)}&\leq C \bigg{(} 1+\int_{0}^{t}  \max_{\tau \in[0,s]} \mathbb{E} \bigg{[} \norm{\hat{u}^{N}(\tau)}_{L^{2}(U)}^2 \bigg{]} ds \bigg{)},\nonumber
\end{align}
concluding the proof of the lemma.

\end{proof}

\section{Well-posedness of limit PDE system, proof of Theorem \ref{T:WellPosednessLimitSystem}} \label{S:WellPosedness}

Let us define the linear operator $T_B: L^2 (U) \rightarrow L^2(U)$ where
\begin{eqnarray}
[T_B u](x) = \int_{U} B(x, x';\mu_{0}) u(x') dx'.\nonumber
\end{eqnarray}

The goal of this section is to prove well-posedness for the pair $(u^{\ast},\hat{u}^{\ast})$ solving the system of equations
\begin{align}
A u^{\ast} (t,x) &= g_{t}(x), \quad x\in U\nonumber\\
u^{\ast}(t,x)&= 0,\quad x\in \partial U,\label{Eq:LimitPDE1_WP}
\end{align}
and
\begin{align}
A^{\dagger} \hat u^{\ast}(t,x)&= ( \la u^{\ast}(t), m \ra - \la h, m \ra  )^{\top} m(x), \quad x\in U\nonumber\\
\hat u^{\ast}(t,x)&= 0,\quad x\in \partial U, \label{Eq:LimitPDE2_WP}
\end{align}
where
\begin{eqnarray}
g_t(x) &=&-\int_{0}^{t} [T_B \hat{u}^{\ast}](s,x) ds\label{Eq:Limiting_g_WP}
\end{eqnarray}

\begin{proof}[Proof of Theorem \ref{T:WellPosednessLimitSystem}]
We first prove that any solution to (\ref{Eq:LimitPDE1_WP})-(\ref{Eq:LimitPDE2_WP}) is unique. Let's suppose there are two solutions $(u^{1,\ast}, \hat u^{1,\ast})$ and $(u^{2,\ast}, \hat u^{2,\ast})$. Let $g^{i}$ be the solution to (\ref{Eq:Limiting_g_WP}) corresponding to $\hat{u}^{i,\ast}$. Define
\begin{eqnarray}
\phi(t,x) &=& u^{1,\ast}(t,x) - u^{2,\ast}(t,x), \notag \\
\psi(t,x) &=& \hat u^{1,\ast}(t,x) - \hat u^{2,\ast}(t,x), \notag \\
\zeta(t,x) &=& g^1_{t}(x) - g^2_{t}(x) .\nonumber
\end{eqnarray}

Then, $(\phi, \psi, \zeta)$ satisfy
\begin{eqnarray}
A \phi(t,x) &=& \zeta(t,x), \notag \\
A^{\dagger}  \psi(t,x) &=&  ( \la \phi(t), m \ra  )^{\top} m(x), \notag \\
\frac{d \zeta}{dt}(t,x) &=& - [T_B \psi](t,x),
\label{LimitPDEdiff}
\end{eqnarray}
with initial condition $\zeta(0,x) = 0$.

By standard elliptic systems made possible due to Assumption \ref{A:Assumption0}, we have the bounds
\begin{eqnarray}
\norm{ \phi(t) }_{H^1_{0}(U)} &\leq& C \norm{ \zeta(t) }_{L^2(U)}, \notag \\
\norm{ \psi(t) }_{H^1_{0}(U)} &\leq& C \norm{ \phi(t) }_{H^1_{0}(U)}.
\label{LaxMilBound}
\end{eqnarray}

Let's now study the evolution of $ \frac{1}{2} (\zeta(t), \zeta(t) )$. Due to the boundedness of the domain $U$ and consequently the boundedness of $B(x, x';\mu_{0})$ we obtain
\begin{eqnarray}
\frac{d}{dt} \frac{1}{2} (\zeta(t), \zeta(t) ) &=& (\zeta(t), \frac{d \zeta}{dt}(t) ) \notag \\
&=& - \int_{U\times U}   \zeta(t,x) B(x, x';\mu_{0}) \psi(t,x')  dx'  dx \notag \\
&\leq& C \norm{ \zeta(t) }_{L^2(U)} \norm{\psi(t) }_{L^2(U)},\nonumber
\end{eqnarray}
where we have used the Cauchy--Schwarz inequality in the last line.

Using the inequality (\ref{LaxMilBound}),
\begin{eqnarray}
\frac{d}{dt} \frac{1}{2} (\zeta(t), \zeta(t) ) &\leq& C \norm{ \zeta(t) }_{L^2(U)} \norm{\psi(t) }_{L^2(U)} \notag \\
&\leq& C \norm{ \zeta(t) }_{L^2(U)} \norm{ \zeta(t) }_{L^2(U)} \notag \\
&=& C ( \zeta(t), \zeta(t) ).\nonumber
\end{eqnarray}

By Gronwall's inequality, for $t \geq 0$,
\begin{eqnarray}
\norm{\zeta(t)}_{L^2(U)} = 0.\nonumber
\end{eqnarray}

Due to the inequalities (\ref{LaxMilBound}), for $t \geq 0$,
\begin{eqnarray}
\norm{ \phi(t) }_{H^1_0(U)} &=& 0, \notag \\
\norm{ \psi(t) }_{H^1_0(U)} &=& 0.\nonumber
\end{eqnarray}
Therefore, we have proven that any solution $(u^{\ast}, \hat u^{\ast})$ to the PDE system (\ref{Eq:LimitPDE1_WP})-(\ref{Eq:LimitPDE2_WP}) must be unique.

Let us now show existence. We will use a standard fixed point theorem argument. For given $T>0$, let $U_{T}=U\times[0,T]$ and $V_{T}=C([0,T];H^{1}_{0}(U))$ be the Banach space consisting of elements of $H^{1}_{0}(U_{T})$ having a finite norm
\[
\norm{v}_{V_{T}}=\max_{t\in[0,T]} \norm{v(t)}_{H^{1}_{0}(U)}.
\]

Let $T>0$ be given and $\ell\in L^{2}([0,T]\times U)$  be a given function. Consider the system of PDEs
\begin{align}
A w (t,x) &= -\int_{0}^{t}\ell(s,x)ds, \quad x\in U\nonumber\\
w(t,x)&= 0,\quad x\in \partial U,\label{Eq:LimitPDE1_WP_E}
\end{align}
and
\begin{align}
A^{\dagger} \hat w(t,x)&= ( \la w(t) - h, m \ra  )^{\top} m(x) , \quad x\in U\nonumber\\
\hat w(t,x)&= 0,\quad x\in \partial U.\label{Eq:LimitPDE2_WP_E}
\end{align}

Due to Assumption \ref{A:Assumption0}, the classical Lax--Milgram theorem guarantees that there is a solution system $(w(t),\hat w(t))\in H^{1}_{0}(U)\times H^{1}_{0}(U) $ of (\ref{Eq:LimitPDE1_WP_E})-(\ref{Eq:LimitPDE2_WP_E}) for all $t\in[0,T]$. As a matter of fact this means that $(w,\hat w)\in V_T\times V_T $.

Now for an arbitrarily given $\hat{u}^{\ast}(t)\in H^{1}_{0}(U)$, let us set $\ell(t,x)=[T_{B}\hat u^{\ast}](t,x)$. Notice that Cauchy--Schwarz inequality gives the estimate
\begin{align}
\norm{\ell(t)}_{L^{2}(U)}&\leq \int_{U}  ([T_{B}\hat u^{\ast}](t,x))^{2}dx\nonumber\\
&\leq    \int_{U}\left(  \int_{U} B(x,x';\mu_{0})\hat{u}^{\ast}(t,x')dx' \right)^{2}dx\nonumber\\
&\leq    \norm{B(\cdot,\cdot;\mu_{0})}_{L^{2}(U\times U)} \norm{\hat u^{\ast}(t)}_{L^{2}(U)}\nonumber\\
&\leq C\norm{\hat u^{\ast}(t)}_{L^{2}(U)}\nonumber
\end{align}

Consider now the map $L: V_T\times V_T\mapsto V_T\times V_T$ by setting $L(u^{\ast},\hat{u}^{\ast})=(w,\hat{w})$.  For given $M<\infty$ and $T<\infty$, let us further define the space
\[
V_{T}(M)=\{u\in V_{T} \text{ and }\norm{u}_{V_{T}}\leq M\}.
\]

We first show that $L$ maps $V_{T}(M)\times V_{T}(M)$ into itself for appropriate choices of $T$ and $M$. By standard elliptic estimates implied by Assumptions \ref{A:Assumption0} and \ref{A:AssumptionTargetProfile} we get for some unimportant constant $C_{1}<\infty$ that
\begin{align}
\norm{w}_{V_T}+\norm{\hat w}_{V_T}&\leq C_{1} T\left[\norm{ u^{\ast}}_{V_T}+\norm{ \hat u^{\ast}}_{V_T}\right]+C_{1}\norm{h}_{L^{2}(U)}\label{Eq:LocalBound}
\end{align}

This means that if $u^{\ast},\hat{u}^{\ast}\in V_{T}(M)$, then $w,\hat{w}\in V_{T}(M)$ if
\[
2C_{1}TM+C_{1}\norm{h}_{L^{2}(U)}\leq M
\]

Such a choice does exist and in particular we choose $T,M$ so that $0<T<\frac{M-C_{1}\norm{h}_{L^{2}(U)}}{2C_{1}M}$. Let's denote such a time $T$ by $T_{1,1}$ and such $M=M_{1}$.

Next, we claim that if $T<\infty$ is small enough, then $L$ is a contraction mapping. For this purpose, consider two pairs $(u_{1}^{\ast},\hat{u}_{1}^{\ast})$, $(u_{2}^{\ast},\hat{u}_{2}^{\ast})$ and their corresponding $(w_{1},\hat{w}_{1})$ and $(w_{2},\hat{w}_{2})$ through the mapping $L$. We have that
\begin{align}
A (w_{1}-w_{2}) (t,x) &= -\int_{0}^{t}[T_{B}(\hat{u}_{1}^{\ast}-\hat{u}_{2}^{\ast})](s,x)ds, \quad x\in U\nonumber\\
w_1(t,x)&=w_2(t,x)= 0,\quad x\in \partial U,\nonumber
\end{align}
and
\begin{align}
A^{\dagger} (\hat w_{1}-\hat w_{2})(t,x)&= ( \la w_1(t) - w_2(t), m \ra  )^{\top} m(x), \quad x\in U\nonumber\\
\hat w_1(t,x)&= \hat w_2(t,x)=0,\quad x\in \partial U.\nonumber
\end{align}

By standard elliptic estimates implied by Assumptions \ref{A:Assumption0} and \ref{A:AssumptionTargetProfile} we get for some unimportant constant $C_{2}<\infty$ that may change from line to line
\begin{align}
\norm{\hat w_1-\hat w_2}_{V_T}&\leq C_{2} \norm{ w_1-w_2}_{V}\nonumber\\
&\leq C_{2} \int_{0}^{T}\norm{ [T_{B}(\hat u_1^{\ast}-\hat u_2^{\ast})](s)}_{L^2_{0}(U)}ds\nonumber\\
&\leq C_{2} \int_{0}^{T}\norm{ (\hat u_1^{\ast}-\hat u_2^{\ast})(s)}_{L^{2}_{0}(U)}ds\nonumber\\
&\leq C_{2} T \norm{ (\hat u_1^{\ast}-\hat u_2^{\ast})}_{V_T}\nonumber
\end{align}

The latter implies that there exists a constant $C_{2}<\infty$ such that
\begin{align}
\norm{w_1- w_2}_{V_T}+\norm{\hat w_1-\hat w_2}_{V_T}&\leq C_{2} T\left[\norm{ u_1^{\ast}-u_2^{\ast}}_{V_T}+\norm{ \hat u_1^{\ast}-\hat u_2^{\ast}}_{V_T}\right]\label{Eq:ContractionProperty}
\end{align}

This means that the mapping $L$ is indeed a contraction on the space $V_{T}\times V_{T}$ if $T$ is chosen such that $C_{2}T<1$. Let's denote such a time $T$ by $T_{1,2}$.

Let's set now $T_{1}=\min\{T_{1,1}, T_{1,2}\}$. Hence, by Banach's fixed point theorem there is a solution $(u^{\ast},\hat{u}^{\ast})$ on $V_{T_{1}}(M_{1})\times V_{T_{1}}(M_{1})$.

Therefore, we have established the existence of a local in time solution to the system (\ref{Eq:LimitPDE1_WP})-(\ref{Eq:LimitPDE2_WP}). Let us now extend this \textcolor{black}{to a solution in $[0,T]$ for an arbitrary $T<\infty$}.

In order to proceed to the next time interval, fixing $M_1$ and $T_1$ be the upper bound and the length of the time interval from the first iteration, the second iteration (here we use that for $[0,T_1]$ we have a solution bounded in the respective norm by $M_1$) gives that we need to choose $T^*$ and $M^*$ such that
    \begin{align*}
    0<T_1&<T^*\nonumber\\
    2C_{1}T_1 M_1+2C_1(T^*-T_1)M^*+C_{1}\norm{h}_{L^{2}(U)}&\leq M^*\nonumber\\
    C_2(T^*-T_1)&<1\nonumber\\
    \end{align*}
    So can set $T_{2,1}$, $T_{2,2}$ and $M_2$ such that
    \begin{align*}
    0&<T_{2,1}-T_1<\frac{M_2-C_1 \norm{h}_{L^{2}(U)}-2C_1 T_1 M_1}{2C_1 M_2}\nonumber\\
    0&<T_{2,2}-T_1<\frac{1}{C_2}
    \end{align*}
and then let $T_2=\min\{T_{2,1},T_{2,2}\}$. Clearly, one can choose $M_2>M_1$ and by Banach's fixed point theorem there will be a solution in $V_{T_2}(M_2)\times V_{T_2}(M_2)$.  Then, we proceed iteratively the same way.

In general to extend the solution to the time interval $[T_k,T_{k+1}]$, we shall set
 $T_{k+1,1}$, $T_{k+1,2}$ and $M_{k+1}$ such that
    \begin{align*}
    0&<T_{k+1,1}-T_k<\frac{M_{k+1}-C_1 \norm{h}_{L^{2}(U)}-2C_1 T_k M_k}{2C_1 M_{k+1}}\nonumber\\
    0&<T_{k+1,2}-T_k<\frac{1}{C_2}
    \end{align*}

The latter display shows that given that one has constructed $(T_j, M_{j})_{j\leq k}$, a viable choice for $M_{k+1}$ is
\[
M_{k+1}=2\left(M_k+2C_1 T_{k-1}(M_k-M_{k-1})\right)>0
\]
which then shows that $T_{k+1,1}$ can be chosen so that $ 0<T_{k+1,1}-T_k<\frac{1}{4C_1}$. Then let $T_{k+1}=\min\{T_{k+1,1},T_{k+1,2}\}$. Clearly, one can always choose $M_{k+1}>M_k$ and by Banach's fixed point theorem there will be a solution in $V_{T_{k+1}}(M_{k+1})\times V_{T_{k+1}}(M_{k+1})$.Notice also that we can establish a global upper bound for the $\norm{\cdot}_{H^{1}_{0}(U)}-$norm of solutions. Indeed, we first have that
\begin{align}
\norm{g_{t}}^{2}_{L^2(U)}&\leq \int_{U}\left(\int_{0}^{t}  [T_{B}\hat u^{\ast}](s,x)ds\right)^{2}dx\nonumber\\
&\leq    \int_{U}\left(\int_{0}^{t}  \int_{U} B(x,x';\mu_{0})\hat{u}^{\ast}(s,x')dx' ds\right)^{2}dx\nonumber\\
&\leq    C \left(\int_{0}^{t}  \int_{U} |\hat{u}^{\ast}(s,x')|dx' ds\right)^{2}\nonumber\\
&\leq    C \left(\int_{0}^{t}  \norm{\hat{u}^{\ast}(s)}_{L^{1}(U)} ds\right)^{2}\nonumber
\end{align}
where we used the boundedness of the kernel $B(x,x';\mu_{0})$ and of $U$.
From this we obtain that for some constant $C<\infty$
\begin{align}
\norm{g_{t}}_{L^2(U)}
&\leq    C \int_{0}^{t}  \norm{\hat{u}^{\ast}(s)}_{L^{1}(U)} ds\nonumber
\end{align}

Going back to the elliptic estimates guaranteed by Assumption \ref{A:Assumption0}  we subsequently obtain
\begin{align}
\norm{\hat u^{\ast}(t)}_{H^{1}_{0}(U)}&\leq C \left(\norm{ u^{\ast}(t)}_{H^{1}_{0}(U)}+ \norm{h}_{L^{2}(U)}\right)\nonumber\\
&\leq C \left(\norm{ g_{t}}_{L^2(U)}+ \norm{h}_{L^{2}(U)}\right)\nonumber\\
&\leq C \left( \int_{0}^{t} \norm{\hat u^{\ast}(s)}_{L^{1}(U)}ds + \norm{h}_{L^{2}(U)}\right)\nonumber\\
&\leq C \left( \int_{0}^{t} \norm{\hat u^{\ast}(s)}_{H^{1}_{0}(U)}ds + \norm{h}_{L^{2}(U)}\right)\nonumber
\end{align}

By Grownall lemma now we have
\begin{align}
\norm{\hat u^{\ast}(t)}_{H^{1}_{0}(U)}
&\leq \norm{h}_{L^{2}(U)} e^{Ct},\nonumber
\end{align}
and as a matter of fact we have the global upper bound for some constants $C_1,C_2<\infty$
\begin{align}
\norm{u^{\ast}(t)}_{H^{1}_{0}(U)}+\norm{\hat u^{\ast}(t)}_{H^{1}_{0}(U)}
&\leq C_{1}\norm{h}_{L^{2}(U)} \left(e^{C_{2}t}+1\right).\nonumber
\end{align}

Thus for any given $T<\infty$ one has
\begin{align*}
\sup_{0\leq t< T}\left(\norm{u^{\ast}(t)}_{H^{1}_{0}(U)}+\norm{\hat u^{\ast}(t)}_{H^{1}_{0}(U)}\right)&<\infty
\end{align*}

Notice that this a-priori global upper bound immediately implies that if the maximal interval of existence of the solution $(u^{\ast},\hat{u}^{\ast})$, say $[0,T^*)$ were to be bounded, i.e., if $T^{*}<\infty$, then $(u^{\ast},\hat{u}^{\ast})$ would be uniformly bounded in $H^{1}_{0}(U)$ in $[0,T^*]$ and would not explode in finite time, see \cite{Amann1985,Nolen}.

The previous construction of the sequence $(T_k,M_k)_{k\geq 1}$ and the aforementioned a-priori bound say that the solution can be extended in time for any time interval $[0,T]$ with $T<\infty$ (see also Section 9.2 of \cite{Evans}, Proposition 6.1 of \cite{Nolen}  and Section 5 of \cite{Amann1985} for related results). This concludes the proof of the theorem.

\end{proof}

\section{Proof of Theorem \ref{T:ConvergenceNTheorem}}\label{S:ProofConvergencePDE}

We will now show that the solution to (\ref{Eq:PreLimitPDE1})-(\ref{Eq:PreLimitPDE2}) converges to the solution to (\ref{Eq:LimitPDE1})-(\ref{Eq:LimitPDE2}) as $N\rightarrow\infty$. We start with the following uniform boundedness lemma with respect to N for the system (\ref{Eq:PreLimitPDE1})-(\ref{Eq:PreLimitPDE2}).
\begin{lemma}\label{L:UniformBoundedness}
Let $T>0$ be given and assume the Assumption \ref{A:Assumption0} holds. Then there exists a constant $C<\infty$ independent of $N$ such that
\begin{align}
\max_{t\in[0,T]}\mathbb{E} \bigg{[} \norm{u^{N}(t)}_{H^{1}_{0}(U)}^2 \bigg{]} +\max_{t\in[0,T]}\mathbb{E} \bigg{[} \norm{\hat u^{N}(t)}_{H^{1}_{0}(U)}^2 \bigg{]} &\leq C.\nonumber
\end{align}
\end{lemma}

\begin{proof}[Proof of Lemma \ref{L:UniformBoundedness}]
By standard energy estimates based on the coercivity property implied by Assumption \ref{A:Assumption0} we have that there is a constant $C_0$ that may change from line to line such that
\begin{eqnarray}
\max_{s\in[0,t]}\mathbb{E} \norm{\hat{u}^{N}(t)}_{H^{1}_{0}(U)}^2 &\leq& C_{0}\left[\max_{s\in[0,t]}\mathbb{E}\norm{ u^{N}(s)}_{H^{1}_{0}(U)}^2 +\norm{ h}_{L^{2}(U)}^2   \right] \notag \\
&\leq& C_{0}\left[\max_{s\in[0,t]}\mathbb{E}\norm{ g^N_s}_{L^2(U)}^2+\norm{ h}_{L^{2}(U)}^2 \right] \notag \\
&\leq& C_{0}\left[1+\int_{0}^{t} \max_{\tau \in [0,s]} \mathbb{E} \norm{\hat{u}^{N}(\tau)}_{H^{1}_{0}(U)}^2 ds+\norm{ h}_{L^{2}(U)}^2 \right],
\label{uHatBoundMay}
\end{eqnarray}
where the third line uses Lemma \ref{L:Bound_g}. The bound for $\displaystyle \max_{t\in[0,T]}\mathbb{E}\norm{\hat{u}^{N}(t)}_{H^{1}_{0}(U)}^2$ then follows by Gronwall's lemma using the fact that, by assumption, $h\in L^{2}$. The bound for $\displaystyle \max_{t\in[0,T]}\mathbb{E}\norm{u^{N}(t)}_{H^{1}_{0}(U)}^2$ follows directly from equation (\ref{uHatBoundMay}) and the bound for $\displaystyle \max_{t\in[0,T]}\mathbb{E}\norm{\hat{u}^{N}(t)}_{H^{1}_{0}(U)}^2$.
This concludes the proof of the lemma.
\end{proof}

Now we can proceed with the proof of Theorem \ref{T:ConvergenceNTheorem}.

\begin{proof}[Proof of Theorem \ref{T:ConvergenceNTheorem}]

Recall that the parameters $c^i_t$ satisfy the equation:
\begin{align}
c^{i}_{t}&=c^{i}_{0}-\frac{\alpha}{N^{1-\beta}}\int_{0}^{t}\int_{U}\hat{u}^{N}(s,x)\sigma(w^{i}_{s} x+\eta_s^i)dxds.\nonumber
\end{align}
Re-arranging and using the Cauchy--Schwarz inequality yields:
\begin{eqnarray}
| c^{i}_{t} - c^{i}_{0} | &\leq&  \frac{\alpha}{N^{1-\beta}}\int_{0}^{t}\int_{U} | \hat{u}^{N}(s,x) | |\sigma(w^{i}_{s} x+\eta_s^i) | dxds \notag \\
&\leq&  \frac{C}{N^{1-\beta}}\int_{0}^{t}\int_{U} | \hat{u}^{N}(s,x) | dxds \notag \\
&\leq& \frac{C}{N^{1-\beta}}\int_{0}^{t} \bigg{(} \int_{U} | \hat{u}^{N}(s,x) |^2 dx \bigg{)}^{\frac{1}{2}} ds \notag \\
&=& \frac{C}{N^{1-\beta}}\int_{0}^{t}  \norm{ \hat{u}^{N}(s)}_{L^2(U)}  ds.
\label{CboundNewMay}
\end{eqnarray}

Taking expectations of both sides of the above equation leads to the following uniform bound:
\begin{eqnarray}
\sup_{ i \in \{1, \ldots, N \}, t \in [0, T]}  \mathbb{E} | c^{i}_{t} - c^{i}_{0} | &\leq&   \frac{C}{N^{1-\beta}}\int_{0}^{T}  \mathbb{E}  \norm{ \hat{u}^{N}(s)}_{L^2(U)}  ds \notag \\
&\leq& \frac{C}{N^{1 - \beta}}.
\label{UniformCparticleBound}
\end{eqnarray}

Let us now consider the parameters $w^i_t$. Using the uniform bound $\displaystyle \sup_{N \in \mathbb{N}, i \in \{1, \ldots, N \}, t \in [0,T] } |c_t^i | \leq C$ from Lemma \ref{L:C_Bound} and the Cauchy--Schwarz inequality,
\begin{eqnarray}
| w^{i}_{t} -w^{i}_{0} | &\leq&\frac{C}{N^{1-\beta}}\int_{0}^{t}\int_{U}|\hat{u}^{N}(s,x)||c^{i}_{s}|\sigma'(w^{i}_{s} x+\eta_s^i)x|dxds, \notag \\
&\leq& \frac{C}{N^{1-\beta}}\int_{0}^{t}\int_{U}|\hat{u}^{N}(s,x)| dxds \notag \\
&\leq& \frac{C}{N^{1-\beta}}\int_{0}^{t}  \norm{ \hat{u}^{N}(s)}_{L^2(U)}  ds.
\label{WboundNewMay}
\end{eqnarray}

Taking expectations, we have the uniform bound
\begin{eqnarray}
\sup_{i \in \{1, \ldots, N \}, t \in [0, T]}  \mathbb{E} | w^{i}_{t} - w^{i}_{0} | &\leq& \frac{C}{N^{1 - \beta}}.
\label{UniformWparticleBound}
\end{eqnarray}
A similar uniform bound holds for the parameters $\eta_t^i$.

Let $f \in C_b^2(\mathbb{R}^{2+d})$. Then, using the global Lipschitz property of $f$,
\begin{eqnarray}
| \la f, \mu^N_t - \mu_0 \ra | &\leq& | \la f, \mu^N_t - \mu_0^N \ra | + | \la f, \mu^N_0- \mu_0 \ra | \notag \\
&\leq& \frac{C}{N} \sum_{i=1}^N \bigg{(} | w^i_t - w^i_0 | + | c^i_t - c^i_0 | + | \eta^i_t - \eta^i_0 |  \bigg{)}  + | \la f, \mu^N_0- \mu_0 \ra |.\nonumber
\end{eqnarray}

Combining the uniform bounds with the equation above, taking expectations, using Markov's inequality and the fact that $\mu^N_0\rightarrow \mu_0$ as $N\rightarrow \infty$, proves that $\la f, \mu^N_t \ra \overset{p} \rightarrow \la f, \mu_0 \ra$ as $N \rightarrow \infty$ for each $t \geq 0$.

Let us next turn our attention to (\ref{Eq:SystemConvergence}). Due to the boundedness of $U$ and of $\sigma,\sigma'$, we obtain
\begin{align}
\norm{ B(\cdot,\cdot;\mu^{N}_{s})}_{L^{2}(U\times U)}&\leq C\left[1+\frac{1}{N}\sum_{i=1}^{N}\left| c^{i}_{s}\right|^{2}\right]\nonumber
\end{align}

By Lemma \ref{L:C_Bound} we then have that there is a universal constant $C<\infty$ such that
\begin{align}
\sup_{s\in[0,T]}\norm{ B(\cdot,\cdot;\mu^{N}_{s})}_{L^{2}(U\times U)}&\leq C\label{Eq:BoundB}
\end{align}

Let $(u^{\ast},\hat{u}^{\ast})$ be the unique solution to the system given in Theorem \ref{T:WellPosednessLimitSystem}. Let us define
\begin{align}
\phi^N(t,x) &=  u^N(t,x)-u^{\ast}(t,x), \notag \\
\psi^N(t,x) &=  \hat u^N(t,x)-\hat u^{\ast}(t,x), \notag \\
\zeta^N(t,x) &=  g^N_{t}(x) -g_{t}(x) .\nonumber
\end{align}

We recall from (\ref{Eq:Rep_g}) the representation
\begin{align}
g_t^N(x) &= g_0^N(x) -\int_{0}^{t} \left(\int_{U} \hat u^{N}(s,x') B(x,x';\mu^{N}_{s}) dx'\right) ds\label{Eq:g_prelimitRep}
\end{align}

Then, we write
\begin{align}
\zeta^N(t,x) &=  g^N_{t}(x) -g_{t}(x)\nonumber\\
&=\zeta^N(0,x)-\int_{0}^{t} \int_{U} \hat u^{N}(s,x') \left(B(x,x';\mu^{N}_{s})-B(x,x';\mu_{0})\right) dx' ds \nonumber\\
&\quad-\int_{0}^{t} \int_{U} \psi^{N}(s,x') B(x,x'; \mu_{0}) dx' ds\nonumber
\end{align}

By taking $L^{2}$ bounds we obtain
\begin{align}
\norm{\zeta^N(t)}_{L^{2}(U)} &\leq \norm{\zeta^N(0)}_{L^{2}(U)}+  \int_{0}^{t}  \norm{\hat u^{N}(s)}_{L^{2}(U)} \norm{B(\cdot,\cdot;\mu^{N}_{s}-\mu_{0})}_{L^{2}(U\times U)} ds \nonumber\\
&+ \int_{0}^{t} \norm{\psi^{N}(s)}_{L^{2}(U)} \norm{B(\cdot,\cdot;\mu_{0})}_{L^{2}(U\times U)} ds\label{Eq:Zbound}
\end{align}

Standard elliptic PDE estimates based on coercivity, the Cauchy--Schwarz inequality, and taking expectations yields the following bound:
\begin{eqnarray}
\mathbb{E}\norm{\psi^{N}(t)}_{H^{1}_{0}(U)} &\leq& \mathbb{E}\norm{\psi^{N}(t)}_{L^2(U)} \notag \\
&=& \mathbb{E}\norm{  \la \phi^N(t), m \ra^{\top} m }_{L^2(U)} \notag \\
&\leq& C \mathbb{E}\norm{\phi^{N}(t)}_{L^2(U)} \notag \\
&\leq& C \mathbb{E}\norm{\phi^{N}(t)}_{H^1_0(U)} \notag \\
&\leq& C_{0}\mathbb{E}\norm{\zeta^{N}(t)}_{L^2(U)} \notag \\
&\leq& C \bigg{(} \mathbb{E}\norm{\zeta^N(0)}_{L^{2}(U)}+  \int_{0}^{t}  \mathbb{E} \bigg{[} \norm{\hat u^{N}(s)}_{L^{2}(U)} \norm{B(\cdot,\cdot;\mu^{N}_{s}-\mu_{0})}_{L^{2}(U\times U)} \bigg{]} ds \notag \\
&+& \int_{0}^{t} \mathbb{E}\norm{\psi^{N}(s)}_{H^{1}_{0}(U)} \norm{B(\cdot,\cdot;\mu_{0})}_{L^{2}(U\times U)} ds \bigg{)}.
\end{eqnarray}

Let us now examine the second term on the second line of the above equation. By the Cauchy--Schwarz inequality and Lemma \ref{L:UniformBoundedness},
\begin{eqnarray}
 \mathbb{E} \bigg{[} \norm{\hat u^{N}(s)}_{L^{2}(U)} \norm{B(\cdot,\cdot;\mu^{N}_{s}-\mu_{0})}_{L^{2}(U\times U)} \bigg{]} &\leq&  \mathbb{E} \bigg{[} \norm{\hat u^{N}(s)}_{L^{2}(U)}^2 \bigg{]}^{\frac{1}{2}} \mathbb{E} \bigg{[} \norm{B(\cdot,\cdot;\mu^{N}_{s}-\mu_{0})}_{L^{2}(U\times U)}^{2} \bigg{]}^{\frac{1}{2}} \notag \\
 &\leq& C \mathbb{E} \bigg{[} \norm{B(\cdot,\cdot;\mu^{N}_{s}-\mu_{0})}_{L^{2}(U\times U)}^{2} \bigg{]}^{\frac{1}{2}}.\nonumber
\end{eqnarray}
Let $f(x,x', c, w, \eta) = \alpha  \sigma(w x' + \eta) \sigma(w x + \eta)+    \sigma'(w x' + \eta)   \sigma'(w x  + \eta) c^2 (x'x+1)$. Using a Taylor expansion, the fact that $U$ is a bounded set, and the uniform bound on $c_t^i$ produces
\begin{eqnarray}
\norm{B(\cdot,\cdot;\mu^{N}_{s}-\mu_{0})}_{L^{2}(U\times U)}^{2} &=& \int_{U} \int_{U} \la f(x, x', \cdot) , \mu^N_s - \mu_0 \ra^2 dx dx' \notag \\
&=& \int_{U} \int_{U} \la f(x, x', \cdot) , \mu^N_s - \mu_0^N \ra^2 dx dx'  +  \int_{U} \int_{U} \la f(x, x', \cdot) , \mu^N_0 - \mu_0 \ra^2 dx dx'. \notag \\
&=& \int_{U} \int_{U} \frac{C}{N^2} \bigg{(} \sum_{i=1}^N [ | w_s^i - w_0^i | + | c_s^i - c_0^i | + | \eta_s^i - \eta_0^i | ] \bigg{)}^2 dx dx' \notag \\
&+& \int_{U} \int_{U} \bigg{(} \frac{1}{N} \sum_{i=1}^N ( f(x, x', c_0^i, w_0^i, \eta_0^i) - \la f(x, x', \cdot) , \mu_0 \ra \bigg{)}^2.\nonumber
\end{eqnarray}

The expectation of the last term is bounded since $(c_0^i, w_0^i, \eta_0^i)$ are i.i.d.\ samples from $\mu_0$. The Cauchy--Schwarz inequality and the inequalities (\ref{CboundNewMay}) and (\ref{WboundNewMay}) yield:
\begin{align}
\int_{U} \int_{U} \frac{C}{N^2} \bigg{(} \sum_{i=1}^N [ | w_s^i - w_0^i | + | c_s^i - c_0^i | + | \eta_s^i - \eta_0^i | ] \bigg{)}^2 dx dx' &\leq \int_{U} \int_{U} \frac{C}{N^2} \bigg{(} \sum_{i=1}^N \frac{C}{N^{1-\beta}}\int_{0}^{T}  \norm{ \hat{u}^{N}(t)}_{L^2(U)}  dt \bigg{)}^2 dx dx'  \notag \\
&\leq  \frac{C}{N^{2 - 2 \beta}} \times \bigg{(} \int_{0}^{T}  \norm{ \hat{u}^{N}(s)}_{L^2(U)}  ds \bigg{)}^2 \notag \\
&\leq \frac{C}{N^{2 - 2 \beta}} \int_{0}^{T}  \norm{ \hat{u}^{N}(s)}_{L^2(U)}^2  ds.\nonumber
\end{align}

Therefore, we get
\begin{eqnarray}
\mathbb{E} \bigg{[} \norm{B(\cdot,\cdot;\mu^{N}_{s}-\mu_{0})}_{L^{2}(U\times U)}^{2}  \bigg{]} &\leq& \frac{C}{N^{2 - 2\beta}} \int_{0}^{T} \mathbb{E} \bigg{[} \norm{ \hat{u}^{N}(s)}_{L^2(U)}^2 \bigg{]}  ds + \frac{C}{N}.\nonumber
\end{eqnarray}

Combining these estimates, we now have the bound
\begin{eqnarray}
\mathbb{E}\norm{\psi^{N}(t)}_{H^{1}_{0}(U)} &\leq& C \bigg{(} \mathbb{E}\norm{\zeta^N(0)}_{L^{2}(U)}+   \frac{C}{N} + \int_{0}^{t} \mathbb{E}\norm{\psi^{N}(s)}_{H^{1}_{0}(U)} ds \bigg{)}.\nonumber
\end{eqnarray}

This bound along with $\norm{\zeta^N(0)}_{L^{2}(U)}\overset{p} \rightarrow 0$ (recall by (\ref{Eq:Boundg0}) that $ \mathbb{E}\norm{g^{N}_{0}}^{2}_{L^{2}(U)} \leq \frac{C}{N^{2\beta-1}}\rightarrow 0$) allow us to conclude via Gronwall's lemma that
 \[
 \max_{t\in[0,T]}\mathbb{E}\norm{\psi^N(t)}_{H^{1}_{0}(U)}\rightarrow 0, \text{ as }N\rightarrow\infty.
 \]

 The latter then implies that
 \[
 \max_{t\in[0,T]}\mathbb{E}\norm{\zeta^N(t)}_{L^{2}(U)}\rightarrow 0, \text{ as }N\rightarrow\infty.
 \]

Consequently,  elliptic estimates based on the coercivity of the operator $A$ give
 \[
  \max_{t\in[0,T]}\mathbb{E}\norm{\phi^N(t)}_{H^{1}_{0}(U)}\leq C \max_{t\in[0,T]}\mathbb{E}\norm{\zeta^N(t)}_{L^{2}(U)}\rightarrow 0, \text{ as }N\rightarrow\infty,
 \]
 concluding the proof of the theorem.
\end{proof}

\section{Properties of the limiting system and related preliminary calculations}\label{S:PreliminaryCalc}

Let us define the linear operator $T_B: L^2 (U) \rightarrow L^2(U)$ where
\begin{eqnarray}
[T_B u](x) = \int_{U} B(x, x';\mu_{0}) u(x') dx'.\nonumber
\end{eqnarray}

Due to the fact that the kernel $B$ is symmetric and of Hilbert-Schmidt type, we immediately get that the integral operator $T_B$ is a self-adjoint, compact linear operator whose eigenfunctions $\{ e_k \}_{k=1}^{\infty}$ form an orthonormal basis for $L^2(U)$, see for example Chapter 8.5 in \cite{Renardy}.

Recall our objective function $\bar{J}(t) = \frac{1}{2} \la u^{\ast}(t) -h, m \ra^{\top} \la u^{\ast}(t) -h, m \ra$. Differentiating with respect to time yields
\begin{eqnarray}
\frac{d \bar{J}}{d t} &=&  \la u^{\ast}(t) -h, m \ra^{\top} \la \frac{d u^{\ast}}{dt}, m \ra \notag \\
&=& \int_{U}  \frac{d u^{\ast}}{dt}  \la u^{\ast}(t) -h, m \ra^{\top} m(x) dx \notag \\
&=& \int_U \frac{d u^{\ast}}{dt} A^{\dagger} \hat u^{\ast} dx \notag \\
&=& \int_U \frac{d A u^{\ast}}{dt} \hat u^{\ast} dx \notag \\
&=& \int_U \frac{dg}{dt} \hat u^{\ast} dx \notag \\
&=& - \int_U \hat u^{\ast}(x) \int_U B(x, x') \hat u^{\ast}(x') dx' dx \notag \\
&=& - ( \hat u^{\ast}, T_B \hat u^{\ast} ).
\end{eqnarray}

Notice that
\begin{eqnarray}
\frac{1}{2} ( A^{\dagger} \hat u^{\ast}(t), A^{\dagger} \hat u^{\ast}(t) ) &=&  \frac{1}{2} ( \la u^{\ast}(t) -h, m \ra^{\top} m,  \la u^{\ast}(t) -h, m \ra^{\top} m ) \notag \\
&=& \frac{1}{2} \la u^{\ast}(t) -h, m \ra^{\top} ( m,  m^{\top} )  \la u^{\ast}(t) -h, m \ra.
\end{eqnarray}

Differentiating with respect to time yields
\begin{eqnarray}
\frac{d}{dt} ( A^{\dagger} \hat u^{\ast}(t), A^{\dagger} \hat u^{\ast}(t) ) &=& \frac{d}{dt} \la u^{\ast}(t) -h, m \ra^{\top} ( m,  m^{\top} )  \la u^{\ast}(t) -h, m \ra.
\end{eqnarray}

By Assumption \ref{A:AssumptionTargetProfile}, we get due to the orthonormal assumption on the elements $m_\ell$ of the vector $\{m_{\ell}\}_{\ell=1}^{L}$
\begin{eqnarray}
\frac{d}{dt} ( A^{\dagger} \hat u^{\ast}(t), A^{\dagger} \hat u^{\ast}(t) ) &=& \frac{d}{dt} \la u^{\ast}(t) -h, m \ra^{\top} I  \la u^{\ast}(t) -h, m \ra \notag \\
&=&  \frac{d \bar{J}}{dt} \notag \\
&=& - ( \hat u^{\ast}, T_B \hat u^{\ast} )\notag\\
&\leq & 0
\end{eqnarray}

Therefore, for all $t\geq 0$ we will have
\begin{eqnarray}
\frac{1}{2} ( A^{\dagger} \hat u^{\ast}(t), A^{\dagger} \hat u^{\ast}(t) )  &=& \frac{1}{2} ( A^{\dagger} \hat u^{\ast}(0), A^{\dagger} \hat u^{\ast}(0) )  - \int_0^t (  \hat u^{\ast}(s),  [T_{B} \hat u^{\ast}](s)   ) ds.
\label{AdjointInequalityNN}
\end{eqnarray}

Then, using Theorem \ref{T:WellPosednessLimitSystem} and Assumption
\ref{A:AssumptionTargetProfile} in order for the boundary terms to vanish, and then  that assuming $A A^{\dagger}$ is coercive, we obtain
\begin{eqnarray}
( A^{\dagger} \hat u^{\ast}(t), A^{\dagger} \hat u^{\ast}(t) )  &=&  ( \hat u^{\ast}(t), A A^{\dagger} \hat u^{\ast}(t) )  \notag \\
&\geq& \gamma (\hat u^{\ast}(t), \hat u^{\ast}(t) ),\nonumber
\end{eqnarray}
where $\gamma > 0$.

Note that this establishes a uniform bound for $\norm{ \hat u^{\ast}(t) }^2$.
\begin{eqnarray*}
(\hat u^{\ast}(t), \hat u^{\ast}(t) ) &\leq& \frac{1}{\gamma} ( A^{\dagger} \hat u^{\ast}(t), A^{\dagger} \hat u^{\ast}(t) )  \notag \\
&\leq&  \frac{1}{\gamma} ( A^{\dagger} \hat u^{\ast}(0), A^{\dagger} \hat u^{\ast}(0) ) \notag \\
&\leq& C < \infty.
\end{eqnarray*}

Note that the constant $C$ does not depend upon time, and therefore
\begin{eqnarray}
\sup_{ t \geq 0} \norm{ \hat  u^{\ast}(t)} < C.
\label{UhatUniformBound}
\end{eqnarray}

The coercivity condition for $A A^{\dagger}$ holds if $A$ is self-adjoint. Indeed, let us use the representation
\begin{eqnarray}
\hat u^{\ast}(t,x) = \sum_{k=1}^{\infty} c_k(t) v_k(x),\nonumber
\end{eqnarray}
where $v_k \in H_0^1(U)$ are the eigenfunctions of $A$ and form an orthonormal basis of $L^2(U)$ \cite{Evans}. We know that the corresponding eigenvalues $\lambda_k \geq \lambda  > 0$. We will prove the coercivity condition when $A$ is self-adjoint, i.e. $A^{\dagger} = A$. Define
\begin{eqnarray}
\hat u_K^{\ast}(t,x)  = \displaystyle \sum_{k=1}^{K} c_k(t) v_k(x).\nonumber
\end{eqnarray}
Then,
\begin{eqnarray}
 \lim_{K \rightarrow \infty} (  \hat u^{\ast},  A A^{\dagger} \hat u_K^{\ast} ) &=&  \lim_{K \rightarrow \infty} ( AA^{\dagger} \hat u^{\ast}, \hat u_K^{\ast} ).
 \label{Kequality1a}
\end{eqnarray}

Using the Cauchy--Schwarz inequality,
\begin{eqnarray}
| ( AA^{\dagger} \hat u^{\ast}, \hat u_K^{\ast} ) - ( AA^{\dagger} \hat u^{\ast}, \hat u^{\ast} ) | &\leq& | ( |AA^{\dagger} \hat u^{\ast} |, | \hat u_K^{\ast}  - \hat u^{\ast} | ) \notag \\
&\leq&  \norm{ AA^{\dagger} \hat u^{\ast} } \norm{ \hat u_K^{\ast}  - \hat u^{\ast} }.\nonumber
\end{eqnarray}

Note that
\begin{eqnarray}
A A^{\dagger} \hat u^{\ast} &=& A (\la  u - h,m\ra)^{\top} m \notag
\end{eqnarray}
Consequently, $A A^{\dagger} \hat u^{\ast}$ is in $L^2$ as long as $A m_{\ell}$ are in $L^2$ (which is assumed in Assumption \ref{A:AssumptionTargetProfile}).

Since $\displaystyle \lim_{K \rightarrow \infty} \norm{ \hat u_K^{\ast}  - \hat u^{\ast} }_{L^{2}(U)} = 0$,
\begin{eqnarray}
\lim_{K \rightarrow \infty} | ( AA^{\dagger} \hat u^{\ast}, \hat u_K^{\ast} ) - ( AA^{\dagger} \hat u^{\ast}, \hat u^{\ast} ) | = 0.\nonumber
\end{eqnarray}

Therefore,
\begin{eqnarray}
 \lim_{K \rightarrow \infty} (  \hat u^{\ast},  A A^{\dagger} \hat u_K^{\ast} ) &=&  \lim_{K \rightarrow \infty} ( AA^{\dagger} \hat u^{\ast}, \hat u_K^{\ast} )  \notag \\
 &=& ( AA^{\dagger} \hat u^{\ast}, \hat u^{\ast} ).
 \label{Kequality1}
\end{eqnarray}

On the other hand,
\begin{eqnarray}
(  \hat u^{\ast}(t),  A A^{\dagger}  \hat u_K^{\ast}(t) )  &=& (  \hat u^{\ast}(t),  A A^{\dagger} \sum_{k=1}^{K} c_k(t) v_k(x) ) \notag \\
&=& (  \hat u^{\ast}(t),   \sum_{k=1}^{K} c_k (t) \lambda_k^2 v_k(x) ) \notag \\
&=& (  \sum_{k=1}^{\infty} c_k (t) v_k(x) ,   \sum_{k=1}^{K} c_k (t)\lambda_k^2 v_k(x) ) \notag \\
&=&  \sum_{k=1}^{K} c_k^2 (t)\lambda_k^2.\nonumber
\end{eqnarray}

From equation (\ref{Kequality1}), we know that
\begin{eqnarray}
\lim_{K \rightarrow \infty}  (  \hat u^{\ast}(t),  A A^{\dagger}  \hat u_K^{\ast}(t) ) = (  \hat u^{\ast}(t), AA^{\dagger} \hat u^{\ast}(t) ).\nonumber
\end{eqnarray}

Since it is also true that $(  \hat u^{\ast}(t),  A A^{\dagger}  \hat u_K^{\ast} (t)) = \displaystyle \sum_{k=1}^{K} c_k^2(t) \lambda_k^2$,
\begin{eqnarray}
\lim_{K \rightarrow \infty} \sum_{k=1}^{K} c_k^2 (t) \lambda_k^2  = (  \hat u^{\ast} (t), AA^{\dagger} \hat u^{\ast} (t)).\nonumber
\end{eqnarray}

Consequently, since $\lambda_k \geq \lambda$,
\begin{eqnarray}
 (  \hat u^{\ast} (t), AA^{\dagger} \hat u^{\ast} (t)) \geq \lambda^2 \sum_{k=1}^{\infty} c_k^2 (t).\nonumber
\end{eqnarray}

Finally, we know that
\begin{eqnarray}
( \hat u^{\ast} (t), \hat u^{\ast} (t)) &=& ( \sum_{k=1}^{\infty} c_k (t) v_k(x), \sum_{k=1}^{\infty} c_k (t) v_k(x) ) \notag \\
&=&   \sum_{k=1}^{\infty} c_k^2 (t).\nonumber
\end{eqnarray}

Therefore,
\begin{eqnarray}
 (  \hat u^{\ast} (t), AA^{\dagger} \hat u^{\ast} (t)) \geq \lambda^2 ( \hat u^{\ast} (t), \hat u^{\ast} (t) ),\nonumber
\end{eqnarray}
which proves the coercivity condition with $\gamma = \lambda^2$ when $A$ is self-adjoint.

Therefore, using the coercivity property of $A A^{\dagger}$,
\begin{eqnarray}
(\hat u^{\ast}(t), \hat u^{\ast}(t) ) &\leq& \frac{2}{ \gamma} \frac{1}{2} ( A^{\dagger} \hat u^{\ast}(t), A^{\dagger} \hat u^{\ast}(t) )   \notag \\
&\leq& \frac{1}{\gamma} ( A^{\dagger} \hat u^{\ast}(0), A^{\dagger} \hat u^{\ast}(0) )  - \frac{1}{\gamma} \int_0^t (  \hat u^{\ast}(s),B  \hat u^{\ast}(s)   ) ds.
\label{InequalityHatU}
\end{eqnarray}

Recall now the linear operator $T_B: L^2(U) \rightarrow L^2(U)$ and let $e_k$ and $\lambda_k$ be the eigenfunctions and eigenvalues, respectively, of the operator $T_B$.

We will next prove that $\lambda_k > 0$. Consider a function $u$ which is non-zero everywhere (i.e., there is a set with positive Lebesgue measure where the function is non-zero). Then,
\begin{eqnarray}
(u, [T_B u]) &=& \int_{U} \int_{U} u(x) B(x,y;\mu_{0}) u(y) dx dy \notag \\
 &=&  \int_{\Theta} \bigg{(} \int_{U} u(x) \sigma(wx+\eta) dx \bigg{)}^2 \mu_{0}(dcdwd\eta) +\int_{\Theta} \left(\int_{U}    u(x)c \sigma'(w x+\eta)x dx\right)^{2} \mu_{0}(dcdwd\eta)\nonumber\\
 & &+\int_{\Theta} \left(\int_{U}    u(x)c \sigma'(w x+\eta) dx\right)^{2} \mu_{0}(dcdwd\eta)\nonumber
\end{eqnarray}
where  $\mu_{0}(dcdwd\eta)$ assigns positive probability to every set with positive Lebesgue measure.

It is clear now that $(u, [T_B u]) \geq 0$. Suppose $(u, [T_B u])  = 0$. Notice that by \cite{Hornik2}, the fact that $\sigma$ is bounded and non-constant by Assumption \ref{A:Assumption1} implies that $\sigma$ is also discriminatory in the sense of \cite{Cybenko,Hornik2}. Namely, if we have that
\[
\int_{U}u(x)\sigma(w x+\eta)\pi(dx)=0 \text{  for all  }(w,\eta) \in \mathbb{R}^{d+1},
\]
then $u(x)=0$  in $U$. Then, $g(w,\eta) = \int_{U} u(x) \sigma(wx+\eta) dx = 0$ almost everywhere in $\mathbb{R}^{d+1}$ and by continuity and the fact that $\sigma$ is also discriminatory we get that $g(w,\eta)=0$ for all $(w,\eta)\in\mathbb{R}^{d+1}$. Then, $u(x) = 0$  everywhere in $U$ by \cite{Cybenko} . However, this is a contradiction, since $u$ is non-zero everywhere. Therefore, if $u$ is non-zero everywhere,
\begin{eqnarray}
(u, [T_B u]) > 0.
\label{Positivity1}
\end{eqnarray}

Let us now consider the eigenfunction $e_k(x)$ with eigenvalue $\lambda_k$.
\begin{eqnarray}
( e_k, [T_B e_k]) &=& (e_k, \lambda_k e_k) \notag \\
&=& \lambda_k (e_k, e_k) \notag \\
&=& \lambda_k.\nonumber
\end{eqnarray}

Equation (\ref{Positivity1}) implies that
\begin{eqnarray}
( e_k, [T_B e_k])  > 0.\nonumber
\end{eqnarray}

Therefore, we have indeed obtained that $\lambda_k > 0$. In addition, we can prove that $\lambda_k < C_{\sigma} < \infty$, i.e. the eigenvalues are upper bounded.
\begin{eqnarray}
(u, [T_B u]) &=&  \int_{\Theta} \bigg{(} \int_{U} u(x) \sigma(wx+\eta) dx \bigg{)}^2 \mu_{0}(dcdwd\eta) +\int_{\Theta} \left(\int_{U}    u(x)c \sigma'(w x+\eta)x dx\right)^{2} \mu_{0}(dcdwd\eta) \notag \\
& &+\int_{\Theta} \left(\int_{U}    u(x)c \sigma'(w x+\eta) dx\right)^{2} \mu_{0}(dcdwd\eta)\notag\\
&\leq&  \int_{\Theta} \norm{u}^{2}_{L^{2}(U)} \bigg{(} \int_{U} |\sigma(wx+\eta)|^{2} dx \bigg{)} \mu_{0}(dcdwd\eta)+ \int_{\Theta} \norm{u}^{2}_{L^{2}(U)}  \bigg{(} \int_{U} \left|c \sigma'(w x+\eta)x\right|^{2} dx \bigg{)} \mu_{0}(dcdwd\eta) \notag \\
& &+ \int_{\Theta} \norm{u}^{2}_{L^{2}(U)}  \bigg{(} \int_{U} \left|c \sigma'(w x+\eta)\right|^{2} dx \bigg{)} \mu_{0}(dcdwd\eta) \notag \\
&\leq& \int_{\Theta}\left(\int_{U} |\sigma(wx+\eta)|^{2} dx+\int_{U} \left|c \sigma'(w x+\eta)x\right|^{2} dx+\int_{U} \left|c \sigma'(w x+\eta)\right|^{2} dx\right)\mu_{0}(dcdwd\eta) \times  \norm{u}^{2}_{L^{2}(U)} \notag \\
&=& C_{\sigma} \norm{u}^2_{L^{2}(U)},\nonumber
\end{eqnarray}

Since $\sigma(\cdot)$ and $U$ are bounded, we will have that $C_{\sigma} < \infty$. Under this assumption,
\begin{eqnarray}
C_{\sigma} &\geq& ( e_k, [T_B e_k]) \notag \\
&=& \lambda_k.\nonumber
\end{eqnarray}

Let $\hat u(t,x) = \displaystyle \sum_{k=1}^{\infty} c_k(t) e_k(x)$ where $e_k$ are the eigenvectors of the operator $T_B$ (which form an orthonormal basis in $L^2$). The inequality (\ref{InequalityHatU}) produces
\begin{eqnarray}
\sum_{k=1}^{\infty} c_k(t)^2 \leq C_0 - \frac{1}{\gamma}  \sum_{i=1}^{\infty} \int_0^t \lambda_i c_i(s)^2 ds,
\label{InequalityHatU2}
\end{eqnarray}
where $C_0 =  \frac{1}{\gamma} ( A^{\dagger} \hat u^{\ast}(0), A^{\dagger} \hat u^{\ast}(0) )$. Therefore,
\begin{eqnarray}
c_k(t)^2 &\leq& C_0 - \frac{1}{\gamma}  \sum_{i=1}^{\infty} \int_0^t \lambda_i c_i(s)^2 ds \notag \\
&\leq& C_0 -  \frac{\lambda_k}{\gamma} \int_0^t c_k(s)^2 ds.\nonumber
\end{eqnarray}

By Gronwall's inequality,
\begin{align}
\int_0^t c_k(s)^2 ds &\leq C_0 \frac{\gamma}{\lambda_k}\left(1- e^{- \frac{\lambda_k}{\gamma} t }\right).\label{Eq:ControlIntegral0}
\end{align}

Therefore, for any $\epsilon > 0$ and any $t>0$, there exists a $\tau > t$ such that $| c_k(\tau) | < \epsilon$. We can use this fact to prove the following lemma.

\begin{lemma} \label{WeakConvergenceHatULemma}
For any function $f \in L^2$, any $\epsilon > 0$, and any $t >0$, there exists a time $\tau > t$ such that
\begin{eqnarray}
| (f, \hat u^{\ast}(\tau) )| < \epsilon.
\label{WeakConvergenceHatU}
\end{eqnarray}
\end{lemma}

\begin{proof}
Let $e_k$ be the orthonormal basis corresponding to the eigenfunctions of the operator $T_B$. Since $f \in L^2$, it can be represented as
\begin{eqnarray}
f(x) = \sum_{k=1}^{\infty} b_k e_k(x).\nonumber
\end{eqnarray}

For any $\epsilon > 0$, there exists an $M$ such that
\begin{eqnarray}
 \sum_{k=M}^{\infty} b_k^2  < \frac{ \epsilon^2}{2 \norm{f}^2}.\nonumber
\end{eqnarray}

Then,
\begin{eqnarray}
(f, \hat u^{\ast}(t) ) &=& (  \sum_{k=1}^{M} b_k e_k, \hat u^{\ast}(t) ) + (  \sum_{k=M+1}^{\infty} b_k e_k, \hat u^{\ast}(t) ).\nonumber
\end{eqnarray}

Using the Cauchy--Schwarz inequality and the uniform bound (\ref{UhatUniformBound}),
\begin{eqnarray}
\sup_{t \geq 0} \bigg{|} (  \sum_{k=M+1}^{\infty} b_k e_k, \hat u^{\ast}(t) ) \bigg{|} &\leq& \sup_{t \geq 0}  \bigg{(} \sum_{k=M+1}^{\infty} b_k^2 \bigg{)}^{\frac{1}{2}} (\hat u^{\ast}(t), \hat u^{\ast}(t) )^{\frac{1}{2}} \notag \\
&<& \frac{\epsilon}{2}.\nonumber
\end{eqnarray}

For any $t$, there exists a $\tau > t$ such that $| c_k(\tau) | <  \frac{\epsilon}{2M}$. Consequently,
\begin{eqnarray}
 (  \sum_{k=1}^{M} b_k e_k, \hat u^{\ast}(\tau) ) \bigg{|} &\leq&  \sum_{k=1}^{M} | b_k c_k(\tau) | \notag \\
&\leq&  \sum_{k=1}^{M} (1 +  b_k^2) | c_k(\tau) | \notag \\
&\leq&  (1 + \norm{f}^2 ) \sum_{k=1}^M | c_k(\tau) | \notag \\
&<& \frac{\epsilon}{2}.
\end{eqnarray}

Therefore, for any $t > 0$, there exists a $\tau > t$ such that
\begin{eqnarray}
\bigg{|} (f, \hat u^{\ast}(\tau) )  \bigg{|} < \epsilon.
\end{eqnarray}

\end{proof}

Recall the adjoint equation is
\begin{eqnarray}
A^{\dagger} \hat u^{\ast}  &=& (\la u^{\ast} - h,m\ra)^{\top}m.
\label{AdjointRepeat}
\end{eqnarray}

Multiplying by the test function $f$ from the set $\mathcal{F} = \{f \in H_0^1(U): A f \in L^2 \}$ and taking an inner product,
\begin{eqnarray}
(f, A^{\dagger} \hat u^{\ast} ) &=& (f, (\la u^{\ast} - h,m\ra)^{\top}m).\nonumber
\end{eqnarray}

This leads to the equation
\begin{eqnarray}
(A f,  \hat u^{\ast} ) &=& (f, (\la u^{\ast} - h,m\ra)^{\top}m).\nonumber
\end{eqnarray}

Therefore, due to Lemma \ref{WeakConvergenceHatULemma}, we have the following proposition.
\begin{proposition}
For any $\epsilon > 0$, any $f \in \mathcal{F}$, and any $ t > 0$, there exists a time $\tau > t$ such that
\begin{eqnarray}
| (f, (\la u^{\ast}(\tau) - h,m\ra)^{\top}m) | < \epsilon.
\label{EpsilonBound}
\end{eqnarray}
\end{proposition}

Let $f = m_{\ell}$ in equation (\ref{EpsilonBound}). Then, for any $\delta > 0$ and any $ t > 0$, there exists a time $\tau > t$ such that
\begin{eqnarray}
| \la u^{\ast}(\tau) - h,m_{\ell} \ra | < \delta.
\label{DeltaBound}
\end{eqnarray}
Consequently, for any $\epsilon > 0$ and any $ t > 0$, there exists a time $\tau > t$ such that
\begin{eqnarray}
\la u^{\ast}(\tau) - h,m \ra^{\top} \la u^{\ast}(\tau) - h,m \ra < \epsilon.
\label{EpsilonBound2}
\end{eqnarray}

\section{Global Convergence of Pre-limit Objective Function} \label{GlobalConvergencePrelimit}

Define the pre-limit objective error function and limit error functions:
\begin{eqnarray}
Q^N(t) &=& ( \la  u^N(t) - h, m \ra^{\top} \la  u^N(t) - h, m \ra )^{\frac{1}{2}}, \notag \\
Q(t) &=& ( \la  u^{\ast}(t) - h, m \ra^{\top} \la  u^{\ast}(t) - h, m \ra )^{\frac{1}{2}}.
\end{eqnarray}

We can now prove the global convergence of the objective function $J^N(t)$.
\begin{proof}[Proof of Theorem \ref{T:PrelimitGlobalMinimumConvTheorem}]

Due to equation (\ref{EpsilonBound2}), for any $\epsilon > 0$,  there exists a time $\tau> 0$ such that
\begin{eqnarray}
Q(\tau) \leq \frac{\epsilon}{4}.
\end{eqnarray}

Due to Theorem \ref{T:ConvergenceNTheorem}, for any $\kappa > 0$, there exists an $N_0$ such that, for $N \geq N_0$,
\begin{eqnarray}
\mathbb{P} \bigg{[} \norm{ u^{\ast}(\tau) - u^N(\tau) } < \frac{\epsilon}{4 (m^{\top},m)^{\frac{1}{2}} } \bigg{]}  > 1- \kappa.\nonumber
\end{eqnarray}

We will select the following choice for $\kappa$:
\begin{eqnarray}
\kappa = \frac{\epsilon^2}{4 C_J},\nonumber
\end{eqnarray}
where $ \mathbb{E} \bigg{[}  (Q^N(0) )^2 \bigg{]}^{\frac{1}{2} } < C_J$. Let us take a moment to show this term is indeed uniformly bounded. Denoting with $(\cdot,\cdot)$ the usual inner product in $L^{2}$ and using the Cauchy--Schwarz inequality, Young's inequality, and a-priori elliptic PDE bounds based on the coercivity Assumption \ref{A:Assumption0}, we get that there is a constant $C<\infty$ that does not depend on $N$ such that
\begin{eqnarray}
\mathbb{E} \bigg{[}  (Q^N(0) )^2 \bigg{]} &=& \mathbb{E} \bigg{[} \la  u^N(0) - h, m \ra^{\top} \la  u^N(0) - h, m \ra  \bigg{]} \notag \\
&\leq& ( m^{\top}, m ) \mathbb{E} \bigg{[} \la  u^N(0) - h,  u^N(0) - h \ra \bigg{]} \notag \\
&=&  ( m^{\top}, m ) \mathbb{E} \bigg{[}  (u^N(0), u^N(0) ) -2 (h,  u^N(0) ) + ( h, h)  \bigg{]} \notag \\
&\leq& ( m^{\top}, m ) \mathbb{E} \bigg{[}  (u^N(0), u^N(0) ) -2 (h, h)^{\frac{1}{2}}  (u^N(0), u^N(0) )^{\frac{1}{2}} + ( h, h)  \bigg{]} \notag \\
&\leq&  ( m^{\top}, m ) \mathbb{E} \bigg{[}  2 (u^N(0), u^N(0) )  +    2( h, h)  \bigg{]} \notag \\
&\leq&  ( m^{\top}, m ) \mathbb{E} \bigg{[}  2 C (g^N_{0}, g^N_{0} )  +    2( h, h)  \bigg{]}.
\end{eqnarray}

By (\ref{Eq:Boundg0}) in the proof of Lemma \ref{L:Bound_g} we have that $\mathbb{E} \big{[}  \norm{g^N_{0}}^2 \big{]}$ is bounded. Therefore, we indeed have that
\begin{eqnarray}
\mathbb{E} \bigg{[}  (Q^N(0) )^2 \bigg{]} < C_J < \infty.\nonumber
\end{eqnarray}

By the triangle inequality and the Cauchy--Schwarz inequality,
\begin{eqnarray}
Q^N(\tau) &=& ( \la  u^N(\tau) - h, m \ra^{\top} \la  u^N(\tau) - h, m \ra )^{\frac{1}{2}} \notag \\
&=& ( \la  u^N(\tau)  -u^{\ast}(\tau) + u^{\ast}(\tau) - h, m \ra^{\top} \la  u^N(\tau)  -u^{\ast}(\tau) + u^{\ast}(\tau) - h, m \ra )^{\frac{1}{2}} \notag \\
&=& \norm{ \la  u^N(\tau)  -u^{\ast}(\tau) + u^{\ast}(\tau) - h, m \ra }_2 \notag \\
&=& \norm{ \la  u^N(\tau)  -u^{\ast}(\tau), m \ra + \la  u^{\ast}(\tau) - h, m \ra }_2 \notag \\
&\leq& \norm{ \la  u^N(\tau)  -u^{\ast}(\tau), m \ra}_2+ \norm{\la  u^{\ast}(\tau) - h, m \ra }_2 \notag \\
&=&   (m^{\top},m)^{\frac{1}{2}} \norm{  u^N(\tau)  - u^{\ast}(\tau) } + \norm{\la  u^{\ast}(\tau) - h, m \ra } \notag \\
&\leq& (m^{\top},m)^{\frac{1}{2}}  \norm{u^N(\tau) - u^{\ast}(\tau) } + \frac{\epsilon}{4}.
\end{eqnarray}

Consequently, for $N \geq N_0$,
\begin{eqnarray}
\mathbb{P} \bigg{[} Q^N(\tau)  < \frac{\epsilon}{2} \bigg{]}  > 1- \kappa.\nonumber
\end{eqnarray}

Recall that $J^N(t) =\frac{1}{2} \la  u^N(t) - h, m \ra^{\top} \la  u^N(t) - h, m \ra $ is monotonically decreasing in $t$. Consequently, $Q^N(t) = ( \la  u^N(t) - h, m \ra^{\top} \la  u^N(t) - h, m \ra )^{\frac{1}{2}}$ is also monotonically decreasing in $t$. Therefore, for $t \geq \tau$ and $N \geq N_0$,

\begin{eqnarray}
\mathbb{P} \bigg{[} Q^N(t)  < \frac{\epsilon}{2} \bigg{]}  &=& \mathbb{E} \bigg{[}  \mathbf{1}_{ Q^N(t) < \frac{\epsilon}{2} }  \bigg{]} \notag \\
&\geq& \mathbb{E} \bigg{[}  \mathbf{1}_{ Q^N(\tau) < \frac{\epsilon}{2} }  \bigg{]} \notag \\
&=& \mathbb{P} \bigg{[} Q^N(\tau)  < \frac{\epsilon}{2} \bigg{]} \notag \\
&>& 1 - \kappa.\nonumber
\end{eqnarray}

Then, for $t \geq \tau$ and $N \geq N_0$,
\begin{eqnarray}
\mathbb{E}[ Q^N(t) ]  &\leq& \mathbb{E} \bigg{[} Q^N(t) \mathbf{1}_{  Q^N(t)  < \frac{\epsilon}{2}} +  Q^N(t) \mathbf{1}_{  Q^N(t)  \geq \frac{\epsilon}{2}}  \bigg{]}  \notag \\
&\leq& \frac{\epsilon}{2} \mathbb{E} \bigg{[}  \mathbf{1}_{  Q^N(t)  < \frac{\epsilon}{2}} \bigg{]} + \mathbb{E} \bigg{[}  Q^N(t) \mathbf{1}_{  Q^N(t)  \geq \frac{\epsilon}{2}}  \bigg{]}  \notag \\
&\leq&  \frac{\epsilon}{2} + \mathbb{E} \bigg{[}  Q^N(0) \mathbf{1}_{  Q^N(t)  \geq \frac{\epsilon}{2}}  \bigg{]}  \notag \\
&\leq&  \frac{\epsilon}{2}+  \mathbb{E} \bigg{[}  (Q^N(0) )^2 \bigg{]}^{\frac{1}{2} } \mathbb{E} \bigg{[} \mathbf{1}_{  Q^N(t)  \geq \frac{\epsilon}{2}} \bigg{]}^{ \frac{1}{2} } \notag \\
&\leq&  \frac{\epsilon}{2} +  \mathbb{E} \bigg{[}  (Q^N(0) )^2 \bigg{]}^{\frac{1}{2} } \mathbb{P} \bigg{[} Q^N(t)  \geq \frac{\epsilon}{2} \bigg{]}^{ \frac{1}{2} } \notag \\
&\leq&  \frac{\epsilon}{2}  + \sqrt{C_J } \sqrt{\kappa}.\nonumber
\end{eqnarray}

Since we have selected $\kappa = \frac{\epsilon^2}{4 C_J }$,
\begin{eqnarray}
\sup_{t \geq \tau, N \geq N_0} \mathbb{E}[ Q^N(t) ]  &\leq&   \frac{\epsilon}{2}  + \frac{\epsilon}{2} = \epsilon,\nonumber
\end{eqnarray}
concluding the proof of the theorem.

\end{proof}

\section{Extensions}\label{S:Extensions}
In this work we studied global convergence for the optimization problem associated to the cost function given by (\ref{ObjectivePDEIntro}). An interesting direction is to consider the same problem but for the cost function
\begin{eqnarray}
J^N(\theta) = \frac{1}{2}   \norm{u^N - h}^{2}_{L^{2}(U)},
\label{ObjectivePDEExt}
\end{eqnarray}
where $h$ is a target profile. One can view the objective function (\ref{ObjectivePDEExt}) as a strong formulation of (\ref{ObjectivePDEIntro}). In this case, the variable $u$ satisfies the elliptic partial differential equation (PDE)
\begin{align*}
A u^N(x) &= f^N( x, \theta ), \qquad x\in U\nonumber\\
u^N(x)&= 0,\qquad x\in \partial U,
\end{align*}
and the adjoint PDE takes the form
\begin{align*}
A^{\dagger} \hat u^N(x)&=   u^N(x)-h(x), \qquad x\in U\nonumber\\
\hat u^N(x)&= 0,\qquad x\in \partial U.
\end{align*}

The methods of this paper allow us to prove theorems analogous to Theorems \ref{T:WellPosednessLimitSystem} and \ref{T:ConvergenceNTheorem}. Consider the non-local PDE system
\begin{align}
A u^{\ast}(t,x) &= g_{t}(x), \quad x\in U\nonumber\\
u^{\ast}(t,x)&= 0,\quad x\in \partial U,\label{Eq:LimitPDE1L2}
\end{align}
and
\begin{align}
A^{\dagger} \hat u^{\ast}(t,x)&= u^{\ast}(t,x) - h(x) , \quad x\in U\nonumber\\
\hat u^{\ast}(t,x)&= 0,\quad x\in \partial U.\label{Eq:LimitPDE2L2}
\end{align}
with
\begin{align}
g_t(x)
&=-\int_{0}^{t} \left(\int_{U} \hat u^{\ast}(s,x') B(x,x';\mu_{0}) dx'\right) ds.\label{Eq:Limiting_gL2}
\end{align}
The system of equations given by (\ref{Eq:LimitPDE1L2})-(\ref{Eq:LimitPDE2L2})-(\ref{Eq:Limiting_gL2}) is well-posed as our next Theorem \ref{T:WellPosednessLimitSystemL2} shows.

\begin{theorem}\label{T:WellPosednessLimitSystemL2}
Let $0<T<\infty$ be given and assume that Assumptions \ref{A:Assumption0} and \ref{A:AssumptionTargetProfile} hold. Then, there is a unique solution to the system (\ref{Eq:LimitPDE1L2})-(\ref{Eq:LimitPDE2L2}) in $C([0,T];H^{1}_{0}(U))\times C([0,T];H^{1}_{0}(U))$.
\end{theorem}

\begin{theorem} \label{T:ConvergenceNTheoremL2}
Let $0<T<\infty$ be given and assume the Assumptions \ref{A:Assumption0}, \ref{A:AssumptionTargetProfile} and \ref{A:Assumption1} hold. Then, as $N \rightarrow \infty$,
\begin{eqnarray*}
\max_{0 \leq t \leq T} \mathbb{E}\norm{u^N(t)-u^{\ast}(t)}_{H^{1}_{0}(U)} &\rightarrow& 0, \notag \\
\max_{0 \leq t \leq T} \mathbb{E}\norm{\hat u^N(t)-\hat u^{\ast}(t)}_{H^{1}_{0}(U)} & \rightarrow& 0, \notag \\
\max_{0 \leq t \leq T} \mathbb{E}\norm{g^N_{t} -g_{t}}_{L^{2}(U)} &\rightarrow& 0. 
\end{eqnarray*}
In addition, $\la f, \mu^{N}_{t} \ra \overset{p} \rightarrow \la f, \mu_0 \ra$ (convergence in probability) for each $t\geq 0$ for all $f \in C^2_b(\mathbb{R}^{2 + d })$.
\end{theorem}

The proofs of Theorems \ref{T:WellPosednessLimitSystemL2} and \ref{T:ConvergenceNTheoremL2} follow very closely those of Theorems \ref{T:WellPosednessLimitSystem} and \ref{T:ConvergenceNTheorem} and thus will not be repeated here.

We can also prove the time-averaged weak convergence of $\hat u^{\ast}(t,x)$ and $u^{\ast}(t,x)$.
\begin{theorem}
For any function $f \in L^2$ and any $\phi \in \{\psi \in H_0^1(U) : A \psi \in L^2 \}$,
\begin{eqnarray}
\lim_{t \rightarrow \infty} \frac{1}{t} \int_0^t  ( f, \hat u^{\ast}(s) )^2 ds &=& 0, \notag \\
\lim_{t \rightarrow \infty} \frac{1}{t} \int_0^t  ( \phi, u^{\ast}(s) -h )^2 ds &=& 0.\notag
\end{eqnarray}

\end{theorem}
\begin{proof}
Similar to (\ref{Eq:ControlIntegral0}), it can be proven that
\begin{eqnarray*}
\int_0^t c_k(s)^2 ds \leq C_0 \frac{\gamma}{\lambda_k}\left(1- e^{- \frac{\lambda_k}{\gamma} t }\right).
\end{eqnarray*}

Therefore,
\begin{eqnarray}
\lim_{t \rightarrow \infty} \frac{1}{t} \int_0^t c_k(s)^2 ds = 0.
\label{TimeAveConv}
\end{eqnarray}

Since $f \in L^2$, it can be represented as
\begin{eqnarray}
f(x) = \sum_{k=1}^{\infty} b_k e_k(x).\nonumber
\end{eqnarray}

For any $\epsilon > 0$, there exists an $M$ such that
\begin{eqnarray}
 \sum_{k=M}^{\infty} b_k^2  < \frac{ \epsilon^2}{2 \norm{f}}.\nonumber
\end{eqnarray}

Then,
\begin{eqnarray}
(f, \hat u^{\ast}(t) ) &=& (  \sum_{k=1}^{M} b_k e_k, \hat u^{\ast}(t) ) + (  \sum_{k=M+1}^{\infty} b_k e_k, \hat u^{\ast}(t) ).\nonumber
\end{eqnarray}

Using the Cauchy--Schwarz inequality and the uniform bound (\ref{UhatUniformBound}),
\begin{eqnarray}
\sup_{t \geq 0} \bigg{|} (  \sum_{k=M+1}^{\infty} b_k e_k, \hat u^{\ast}(t) ) \bigg{|} &\leq& \sup_{t \geq 0}  \bigg{(} \sum_{k=M+1}^{\infty} b_k^2 \bigg{)}^{\frac{1}{2}} (\hat u^{\ast}(t), \hat u^{\ast}(t) )^{\frac{1}{2}} \notag \\
&<& \frac{\epsilon}{2}.\nonumber
\end{eqnarray}

By the Cauchy--Schwarz inequality,
\begin{eqnarray*}
 \bigg{(}  \sum_{k=1}^{M} b_k e_k, \hat u^{\ast}(t) \bigg{)}^2 &\leq& \bigg{(}  \sum_{k=1}^{M} | b_k c_k(t) |  \bigg{)}^2 \notag \\
&\leq& \norm{f}^2 \sum_{k=1}^M c_k(t)^2.
\end{eqnarray*}

Then,
\begin{eqnarray*}
\frac{1}{t} \int_0^t  ( f, \hat u^{\ast}(s) )^2 ds &\leq& \frac{C}{t} \int_0^t \bigg{(} \sum_{k=1}^M c_k(s)^2 + \epsilon^2   \bigg{)} ds.
\end{eqnarray*}

Due to equation (\ref{TimeAveConv}),
\begin{eqnarray*}
\lim \sup_{t \rightarrow \infty} \frac{1}{t} \int_0^t  ( f, \hat u^{\ast}(s) )^2 ds \leq  C \epsilon^2.
\end{eqnarray*}

Since $\epsilon$ was arbitrary,
\begin{eqnarray*}
\lim_{t \rightarrow \infty} \frac{1}{t} \int_0^t  ( f, \hat u^{\ast}(s) )^2 ds  = 0.
\end{eqnarray*}

Consequently,
\begin{eqnarray*}
\lim_{t \rightarrow \infty} \frac{1}{t} \int_0^t  ( \phi, u^{\ast}(s) -h )^2 ds &=& \lim_{t \rightarrow \infty}  \frac{1}{t} \int_0^t  ( A \phi, \hat u^{\ast}(s) )^2 ds \notag \\
&=& 0.
\end{eqnarray*}

\end{proof}

An open question for future research is to obtain a global convergence theorem for the pre-limit PDE $u^N(t,x)$ analogous to the global convergence result Theorem \ref{T:PrelimitGlobalMinimumConvTheorem}. This theorem would require proving a decay rate for $c_k(s)$, which remains a significant technical challenge. As we also mentioned in Section \ref{SS:Applications}, another interesting future research topic is extending these theoretical results to nonlinear partial differential equations.

\label{Eq:ControlIntegral}

\end{document}